\documentclass{article}

\usepackage[preprint]{neurips_2026}
\usepackage{listings}
\usepackage[utf8]{inputenc} 
\usepackage[T1]{fontenc}    
\usepackage[pdfencoding=unicode,psdextra,citecolor=blue,colorlinks=true]{hyperref}       
\usepackage{url}            
\usepackage{booktabs}       
\usepackage{amsfonts}       
\usepackage{nicefrac}       
\usepackage{microtype}      
\usepackage{xcolor}         
\usepackage{graphicx}
\usepackage{subcaption}

\usepackage{enumitem}

\usepackage{algorithm}
\usepackage{algorithmic}

\usepackage{amsmath}
\usepackage{amssymb}
\usepackage{mathtools}
\usepackage{amsthm}
\usepackage{thmtools}
\usepackage{multirow}
\usepackage{wrapfig}
\usepackage[capitalize,noabbrev]{cleveref}
 
\newtheorem{theorem}{Theorem}
\newtheorem{lemma}[theorem]{Lemma} 
\newtheorem{proposition}[theorem]{Proposition} 
\newtheorem{corollary}[theorem]{Corollary}

\usepackage{nccmath}
\theoremstyle{remark}
\newtheorem{remark}[theorem]{Remark}

\crefname{equation}{eqn.}{eqns.}
\crefname{appendix}{appendix}{appendices}
\Crefname{appendix}{Appendix}{Appendices}

\newcommand{\1}{\mathbf{1}}

\usepackage{amsmath,amsfonts,bm}

\def\eqref#1{equation~\ref{#1}}

\def\1{\bm{1}}

\def\eps{{\epsilon}}

\DeclareMathAlphabet{\mathsfit}{\encodingdefault}{\sfdefault}{m}{sl}
\SetMathAlphabet{\mathsfit}{bold}{\encodingdefault}{\sfdefault}{bx}{n}

\newcommand{\E}{\mathbb{E}}

\newcommand{\Var}{\mathrm{Var}}

\newcommand{\algo}{VSPO}

\definecolor{darkgreen}{RGB}{0,150,0}

\title{\algo: Vector-Steered Policy Optimization for Behavioral Control}

\author{%
  Xuechen Zhang\thanks{These authors contributed equally to this work. Correspondance to \{zxuechen,zijianh,oymak\}@umich.edu} \\
    University of Michigan \\
  \texttt{zxuechen@umich.edu} \\
  \And
  Zijian Huang\footnotemark[1] \\
  University of Michigan \\
\texttt{zijianh@umich.edu}\\
  \And
  Kai Yang \\
  University of Michigan \\
  \texttt{yangkai@umich.edu} \\
  \AND
  Weijia Zhang \\
  University of Michigan \\
  \texttt{wjzhang@umich.edu} \\
  \And
  Jiasi Chen \\
  University of Michigan \\
  \texttt{jiasi@umich.edu}\\
  \And
  Samet Oymak \\
  University of Michigan \\
\texttt{oymak@umich.edu}\\
 \\
}

\begin{document}

\maketitle

\begin{abstract}
Modern language models often need to optimize a primary accuracy objective while also accommodating secondary behavioral preferences, such as verbosity, agreeableness, or the level of technical expertise in its response.
In practice, a base model may exhibit a desired behavior very rarely or not at all. Thus, endowing the model with a target behavior creates a \emph{sparse behavioral reward} bottleneck. To address such multi-objective problems, we introduce Vector-Steered Policy Optimization (VSPO) which employs a steering vector associated with the target behavior to control the behavior intensity of the generated rollouts. VSPO is obtained by modifying GRPO to sample rollouts with varying steering intensities. This process can be interpreted as an on-policy \emph{latent self-distillation} procedure where the model internalizes its steering vector. By varying steering intensities, VSPO \emph{upsamples} rare behaviors and enriches rollout diversity, which alleviates the sparse reward issue and provably accelerates the policy optimization. Through comprehensive theory and experiments, we establish that VSPO has favorable properties compared to vanilla reward shaping and other alternative approaches. Specifically, under a bandit abstraction, VSPO provably achieves better iteration complexity than reward-shaped GRPO when the steering-induced distributions are sufficiently aligned with the target behavior. We evaluate VSPO across multiple reasoning benchmarks, including MATH and MMLU-Pro, for four target behaviors: explanation expertise, confidence expression, robustness to misleading context, and response verbosity. Our results show that VSPO consistently improves the control along target behavior while maintaining or improving task accuracy compared with reward shaping, teacher-trace distillation, and guidance-based baselines.

\end{abstract}

\section{Introduction}
\label{sec:intro}

Recent years witnessed remarkable progress on post-training and reasoning capabilities of large language models (LLM) \cite{ouyang2022training,shao2024deepseekmath}. These models can now tackle highly complex mathematical and scientific problems. However, many desirable model behaviors are not captured by correctness alone. In reasoning tasks, we often want the model to maintain a primary objective like answer correctness, while also controlling secondary characteristics such as verbosity, confidence expression, expertise level, or agreeableness. This setting is closely related to multi-objective alignment, where models must trade off several human-preference dimensions rather than optimizing a single scalar notion of quality. Prior work has explored multi-objective or dynamically adjustable reward formulations \cite{xiongprojection,yang2024rewards}, but such methods still primarily operate at the reward level and inherit the usual difficulties of reward design, scaling, and optimization stability.

\begin{figure}
\centering
\includegraphics[width=\textwidth]{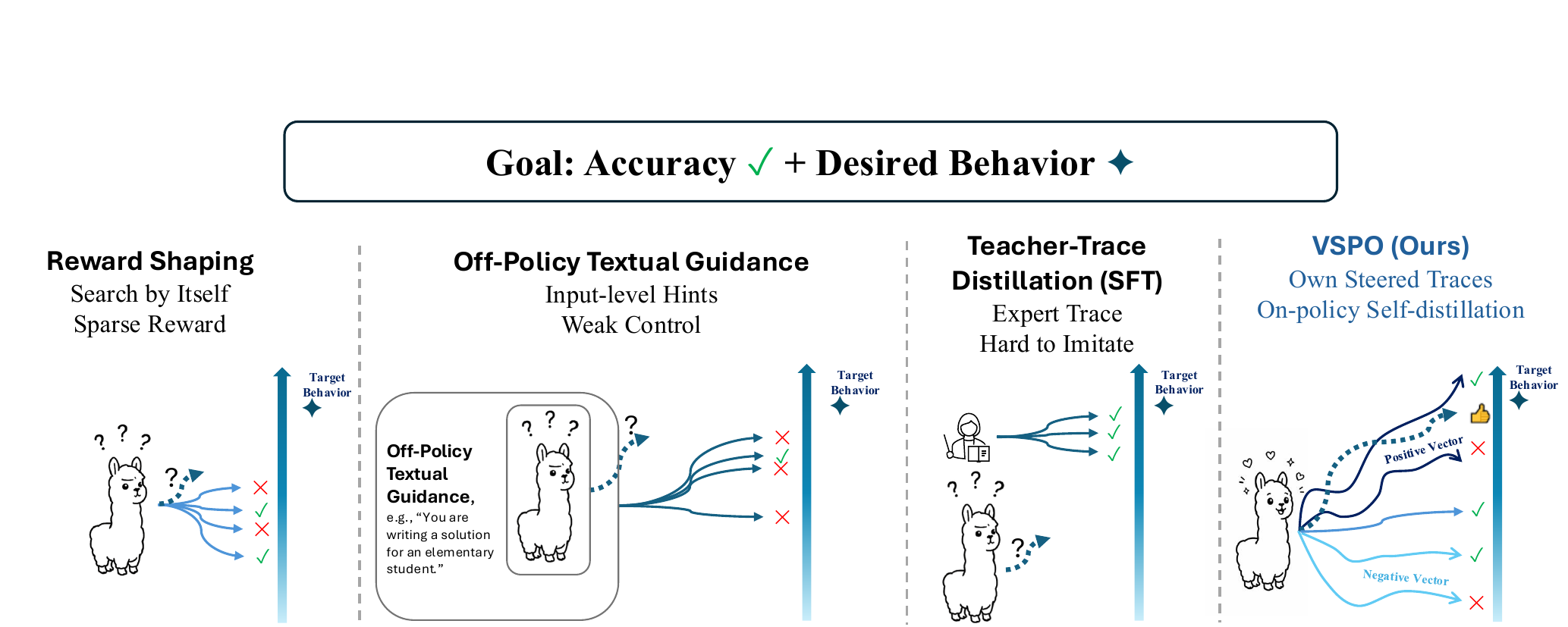}
\caption{\textbf{Motivation and overview of VSPO.}  The goal is to improve task accuracy while inducing a desired reasoning behavior. The figure illustrates VSPO and alternative methods for the goal of improving both task accuracy and a target behavior. Check marks and crosses indicate whether each generated trace is correct or incorrect. The vertical blue arrow denotes the target-behavior direction: traces higher along the arrow better exhibit the desired behavior. VSPO uses latent steering vectors to generate a structured family of the model’s own steered rollouts, spanning different behavior intensities while remaining close to the current policy. By rewarding accurate and behavior-aligned traces and distilling them back into the unsteered policy, VSPO enables on-policy self-distillation toward both correctness and controllable reasoning behavior without requiring inference-time textual guidance. The bold dashed trajectory refers to the rollout generated at inference-time and thumbs-up indicates that, after VSPO training, the model can generate correct target-behavior traces. 
}\label{fig:fig1}
\vspace{-5pt}
\end{figure}

A straightforward approach is to apply GRPO with an auxiliary shaped reward for the desired behavior~\cite{aggarwall1,team2025kimi,yang2024rewards,zhang2025making}. While intuitive and simple, reward shaping can be inefficient when the target behavior is rarely expressed in the sampled rollouts. Scalar rewards only evaluate generated trajectories; they do not directly improve the diversity or quality of the rollout distribution. If the majority of responses in a GRPO group are low-quality or behaviorally similar, relative advantage normalization provides little useful signal for learning the desired characteristic. Another common approach is to use a stronger teacher model to generate or rewrite reasoning traces endowed with the target behavior and then fine-tune the student model on these traces. Teacher traces can provide cleaner demonstrations of the desired behavior, but they are off-policy with respect to the student: the student is trained on trajectories produced by another model rather than on its own sampled outputs. This off-policy mismatch is one of the motivations behind recent on-policy self-distillation methods \cite{shenfeld2026self,hubotter2026reinforcement,zhao2026self}. 
These works suggest that staying on-policy via self generated traces helps reduce the mismatch introduced by external feedback, specifically, teacher-trace. This also inspired our method which could be interpreted as self-distillation in latent steering space.
A third approach is to use powerful teacher models during online RL, for example as AI-feedback annotators, judges, or rewriters to achieve the target behavior. RLAIF \cite{guo2024direct,lee2024rlaif} replaces or augments human feedback with feedback from AI models, and online AI-feedback methods query an LLM annotator during training to obtain preference labels for current-policy samples. These methods improve scalability relative to human annotation and can provide controllable feedback, but repeated online calls to powerful API models are expensive, especially when used for multi-rollout generation or trace rewriting inside the RL loop. Moreover, training the student on rewritten traces still incur off-policy drawbacks \cite{sharma2024critical}.


As depicted in Figure~\ref{fig:fig1}, we propose Vector-Steered Policy Optimization (\algo) to address these bottlenecks. Recent work on steering vectors \cite{chen2025persona} shows that behavioral traits can be represented as latent directions in a model’s activation space, suggesting a way to quantitatively characterize and steer target behaviors beyond surface-level prompting or scalar rewards. \algo~relies on a steering vector associated with the target behavior. {In practice, we construct this vector using a more powerful teacher model to obtain contrastive rewrite pairs associated with the behavior.} During RL, the current policy is rolled out under multiple steering intensities, producing a diverse and controllable rollout distribution. Because these rollouts are still generated by the target model, training remains anchored to the current policy. We then apply GRPO with a multi-objective reward, where the behavioral reward is proportional to the steering intensity. The full procedure is depicted in Figure \ref{fig:algo}.

Overall, our paper makes the following contributions:
\begin{itemize}[leftmargin=*, itemsep=0pt, topsep=0pt, parsep=0pt, partopsep=0pt]
    \item  We introduce \algo, a multi-objective policy optimization method that leverages steering vectors to control language model behavior while preserving task accuracy. Compared to vanilla GRPO or reward shaping, \algo~actively improves rollout diversity by explicitly sampling traces along the target latent vector. Compared with SFT on teacher traces, it avoids directly training on off-policy teacher trajectories: the teacher can assist vector construction, but the RL process is fully guided by the student rollouts. Compared with RLAIF or online rewriting, VSPO substantially reduces reliance on repeated teacher calls, since the steering vector is reused throughout optimization. 
    \item We formally characterize the advantage of VSPO over reward shaping: We establish how \algo~can provably achieve better iteration complexity than reward-shaped GRPO under a bandit abstraction by upsampling sparse but desirable behavioral patterns (see Section \ref{sec:theory}, Propositions \ref{prop:iter-rs}, \ref{prop:iter-lsd}).
    \item We empirically evaluate \algo\ on MATH and MMLU-Pro across four target behavior types (see Section \ref{subsubsec:task} for details) and show that it consistently outperforms reward shaping, teacher-trace SFT, and RLAIF-style baselines while maintaining correctness.
     
    \item VSPO has favorable properties: Compared to alternatives, it incurs mild distributional shift from base model (see Fig.~\ref{fig:nll}). It induces a natural learning curriculum along the behavior axis throughout policy optimization (see Fig.~\ref{fig:review_combine}). The behavioral component of its reward function is normalized by design as it corresponds to \emph{one unit of steering vector}.
\end{itemize}

\section{Related Work}
Teacher-trace distillation is a common technique to guide target models, where a stronger model generates reasoning traces and the target model is trained with supervised fine-tuning. This paradigm has been widely used for improving reasoning and controlling response format \citep{hassid2025don,kang2025c3ot,muennighoff2025s1}. However, teacher-trace SFT is off-policy: the student learns from trajectories generated by an external model rather than from its own sampled rollouts, which can cause a mismatch between the distilled behavior and the student policy's reachable distribution \citep{guo2025deepseek,zhang2025making}. 
RL post-training methods, especially RLVR, address this by optimizing the model on its own sampled responses using verifiable rewards and group-relative advantages \citep{shao2024deepseekmath,guo2025deepseek}. Yet standard GRPO relies on ordinary rollout sampling and scalar rewards, which may be insufficient when the goal is to preserve correctness while controlling secondary characteristics such as conciseness, confidence expression, expertise level, or sycophancy resistance. Another line of work incorporates target characteristics directly into rewards, such as length penalties or LLM-based preference scores, and more broadly studies multi-objective alignment \citep{team2025kimi,aggarwall1,zhang2025making}. However, reward shaping only changes scalar feedback after rollout generation and does not directly improve the diversity of the sampled candidates. In contrast, VSPO targets controllable secondary behaviors by shaping the rollout distribution through latent steering vectors. Similarly, textual guidance and RLAIF-style methods use auxiliary instructions, critiques, or AI feedback to elicit desired behaviors during RL \citep{snell2022learning,yao2026incorporating,yan2025learning,yang2026towards}, but repeated external guidance can be expensive and may introduce mismatch between guided training and unguided inference.
Recent self-distillation methods reduce off-policy mismatch by deriving learning signals from the model itself. SDFT uses a demonstration-conditioned version of the model as a self-teacher, while SDPO uses high-quality self-generated rollouts as implicit optimization targets \citep{shenfeld2026self,hubotter2026reinforcement}. Related work also improves RLVR through refined credit assignment, denser intermediate signals, and multi-attempt verifier feedback \citep{kazemnejadvineppo,Learning_to_Correct}. These methods are related to our motivation, but mainly target task-performance rather than controllable secondary behavior. Finally, recent curriculum-based RL methods have been shown to help overcome optimization bottlenecks in SFT and RL training \citep{zhang2026bread}. VSPO induces a natural curriculum along the behavior axis by varying steering intensity, and we empirically show that this improves over naive applications of SFT and RL for acquiring target behaviors.


Our work is also related to representation engineering and vector steering, which identify semantically meaningful directions in activation space and manipulate model behavior at inference time \citep{zou2023representation,turner2023steering,panickssery2023steering}. Persona vectors further specialize such directions to behavioral traits \citep{chen2025persona}. Unlike prior work that primarily uses steering as a test-time intervention, \algo\ uses vector steering as a mechanism for on-policy self-distillation. The details of related works is at \Cref{app:related}.

\section{Proposed Method: \textbf{V}ector-\textbf{S}teered \textbf{P}olicy \textbf{O}ptimization (\algo)}
\begin{figure}
\centering
\includegraphics[width=\textwidth]{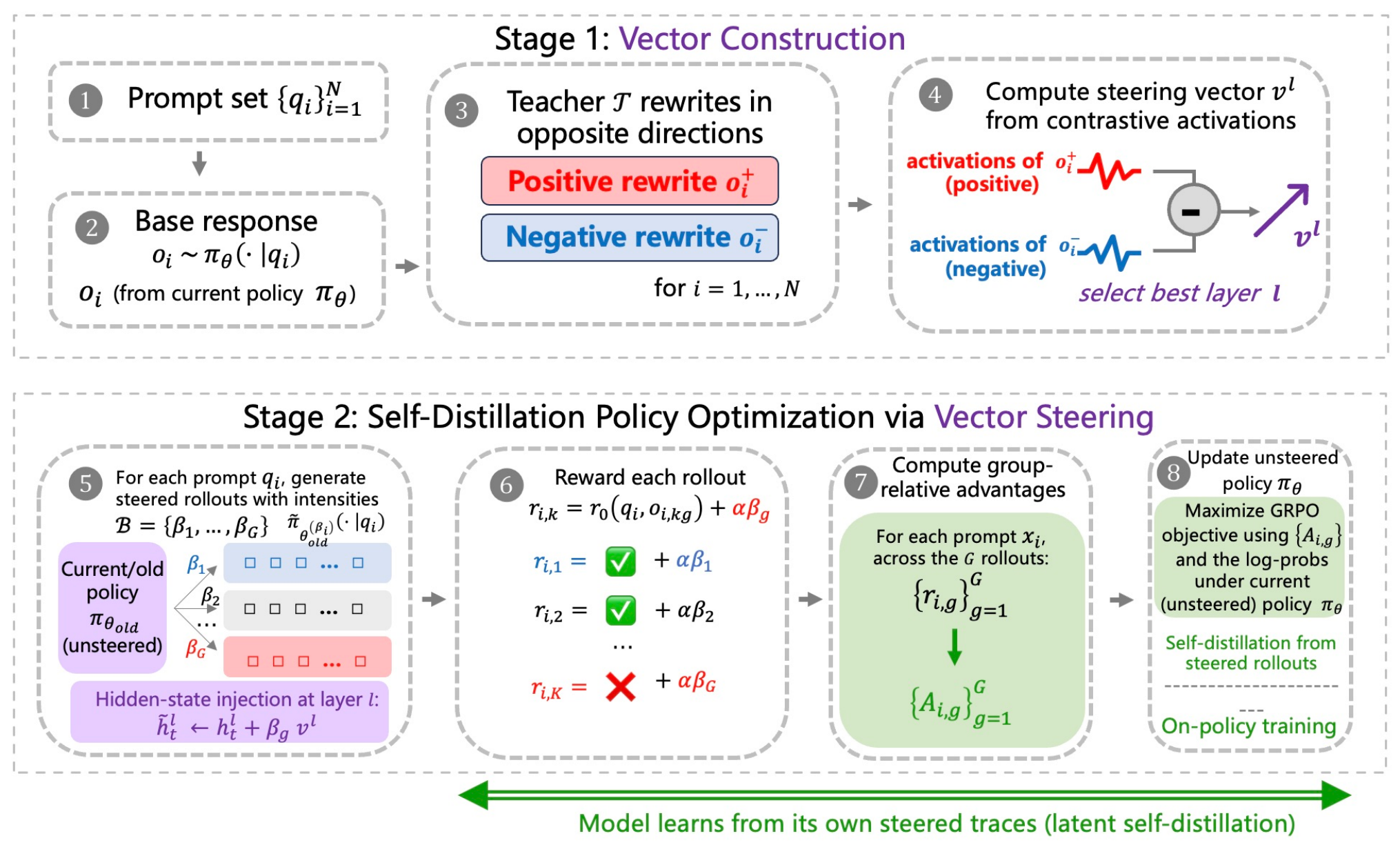}
\vspace{-8pt}
\caption{\textbf{Overview of \algo~algorithm}. In Stage 1, prompts are sampled from the current policy, a teacher rewrites the responses into contrastive positive and negative directions, and their activation differences are used to construct a steering vector. In Stage 2, the current policy generates on-policy rollouts under different steering intensities, receives rewards summing task correctness and steering-dependent preference signals, and is updated with GRPO. The final update is applied to the unsteered policy, enabling the model to distill useful behaviors from its own steered traces.} \label{fig:algo}
\end{figure}


\subsection{Algorithm Overview}
\label{subsec:algo}
The workflow of \algo\ is shown in \Cref{fig:algo} and proceeds as follows:
\begin{itemize}[leftmargin=*, topsep=0pt, parsep=0pt, partopsep=0pt]
    \item \textbf{Steering Vector Construction}: In the first stage, the target model generates base responses for a small set of sampled questions. A teacher model then rewrites each response into positive and negative behavioral variants, such as expert versus elementary explanations for the expertise-level axis. Using the original prompt, we extract target-model activations for both variants, compute their mean activation difference at each layer, and select the strongest layer-wise direction as the steering vector.
    \item \textbf{Self-Distillation Policy Optimization}: In the second stage, for each sampled prompt $q_i$,
    \begin{enumerate}[leftmargin=*, topsep=0pt, parsep=0pt, partopsep=0pt]
        \item \textbf{Rollout with intensities}:  Sample a group of $G$ rollouts $\{o_{i,g}\}_{g=1}^G$ with intensities $\mathcal{B}=\{\beta_g\}_{g=1}^G$.
        \item \textbf{Reward computation}: For each rollout $o_{i,g}$, the reward $r_{i,g}$ is a combination of the primary reward and the proxy behavioral reward, $r_0(q_i,o_{i,g})+\alpha\beta_g$, where $r_0$ is the primary reward function, and $\alpha$ is a steering preference hyperparameter. 
        \item \textbf{Policy updates}: Optimize the policy via the following objective:
        {\footnotesize\begin{equation}
            \mathcal{L}_{\algo}(\theta)
            =
            \mathbb{E}_{x}
            \left[
            \frac{1}{G}\sum_{g=1}^{G}
            \frac{1}{T_g}\sum_{t=1}^{T_g}
            \min\!\left(
            \rho_{g,t}(\theta)A_g,\;
            \mathrm{clip}\!\big(\rho_{g,t}(\theta),1-\epsilon,1+\epsilon\big)A_g
            \right)
            \right]
            -
            \lambda \mathcal{D}_{\mathrm{KL}}(\pi_\theta \,\|\, \pi_{\mathrm{ref}}).
            \label{eq:grpo_vs}
        \end{equation}}
        where $\bar{r},\sigma_r$ are the mean and standard deviation of the rewards in the same group
        {\footnotesize\begin{equation}
            \rho_{g,t}(\theta)
            =
            \frac{\pi_\theta(o_{g,t}\mid q,o_{g,<t})}
            {\pi_{\theta_{\mathrm{old}}}(o_{g,t}\mid q,o_{g,<t})},\quad A_g = \frac{r_g-\bar{r}}{\sigma_r}.
            \label{eq:ratio_and_advantage}
        \end{equation}}
    \end{enumerate}
\end{itemize}
The details of the algorithm can be found in \Cref{alg:popcorn} and \Cref{appendix:method}. 

\subsection{Theoretical Analysis}
\label{sec:theory}
VSPO can be viewed as a form of \emph{distribution shaping} during sampling: Instead of drawing a group entirely from the current policy, it draws from multiple related but contrasted distributions. 
In this section, we quantify the benefit of distribution shaping obtained by vector steering, comparing to reward shaping, which is a typical approach to solving multiple-objective optimization in GRPO. 

\paragraph{Theoretical Setup.}
We consider a $K$-armed bandit with deterministic rewards, i.e., a one-step decision problem with $K$ possible actions. 
Each arm $i\in [K]$ has a primary reward $x(i)$ and behavioral reward $y(i)$. 
The target arm $i^\star$ is defined as $i^\star
=
\arg\max_{i\in[K]}
\left\{
x(i): y(i)=\max_{j\in[K]} y(j)
\right\}$.


Denote $\pi_t$ as policy at iteration $t$. 
Define scalarized reward $r(i) := x(i)+\alpha y(i)$
, and expected reward of policy $\pi$ as $J(\pi) := \sum_{i=1}^K \pi(i) r(i)$. 
We assume $\pi_0(i)>0$ for all $i\in[K]$, and $\alpha>
\max_{i:\, y(i)<y(i^\star)}
\frac{\max\{x(i)-x(i^\star),0\}}
{y(i^\star)-y(i)}$ so that $i^\star$ is the unique maximizer of $r(i)$ (Proof in Appendix \ref{sec:optimal-arm}). 


For each suboptimal arm $i\neq i^\star$, define the scalar reward gap $\Delta_i := r(i^\star) - r(i)$, together with $\Delta_{\min} := \min_{i \neq i^\star} \Delta_i, \; \Delta_{\max} := \max_{i\ne i^\star} \Delta_i$.
We also define the \textit{bandit conditioning constant} $\lambda$ as 
\begin{equation}
\label{eq:def-misalign}
\lambda := \frac{D_x+\alpha D_y}{\Delta_{\max}}, \quad \text{where }
D_x := \max_{i\in[K]} x(i)-\min_{i\in[K]} x(i),\;
D_y := \max_{i\in[K]} y(i)-\min_{i\in[K]} y(i). 
\end{equation}

\paragraph{GRPO.}
At iteration $t$, we draw $G$ i.i.d. samples from the current policy $\pi_t$. 
The reward for each sample $i$ is given by reward shaping $r(i) = x(i)+\alpha y(i)$. 
Let $\bar r_t^{\text{GRPO}}$ and $\bar\sigma_t^{\text{GRPO}}$ be the mean and standard deviation of the sampled scalar rewards. 
The empirical arm-level GRPO score is
\begin{equation}
\label{eq:empirical-rs-score}
\widehat A_t^{\text{GRPO}}(i) :=
{N_{t}(i)}\frac{r_i - \bar r_t^{\text{GRPO}}}{\bar\sigma_t^{\text{GRPO}}},
\end{equation}
where $N_t(i)$ is the number of times arm $i$ appears in the group. 

\paragraph{VSPO.}
At each iteration $t$, we assume applying vector steering returns two distributions $\pi_t^+$ and $\pi_t^-$ satisfying $\pi_t = \frac{\pi_t^+ + \pi_t^-}{2}$. 
VSPO draws $G/2$ samples from $\pi_t^+$ and $G/2$ samples from $\pi_t^-$. 
The shaped reward for samples from $\pi_t^+,\pi_t^-$ are defined as $r_t^+, r_t^-$ respectively:
\begin{equation*}
    r_t^+(i) := x(i)+ \alpha \mu_{y,t}^+, \;r_t^-(i) := x(i)+ \alpha \mu_{y,t}^-,\; \text{where } \mu_{y,t}^+ := \E_{j\sim\pi_t^+}[y(j)],\; \mu_{y,t}^- := \E_{j\sim\pi_t^-}[y(j)].
\end{equation*}
Let $\bar r_t^{\text{VSPO}},\bar\sigma_t^{\text{VSPO}}$ be the mean and standard deviation of all shaped rewards, and $N_t^+(i),N_t^-(i)$ be the number of times arm $i$ appears in the groups $\pi_t^+, \pi_t^-$. 
The empirical arm-level VSPO score is
\begin{equation}
\label{eq:empirical-lsd-score}
\widehat A_t^\text{VSPO}(i)
:=
N_t^+(i)
\frac{r_t^+(i)-\bar r_t^\text{VSPO}}{\bar \sigma_t^\text{VSPO}}
+
N_t^-(i)
\frac{r_t^-(i)-\bar r_t^\text{VSPO}}{\bar \sigma_t^\text{VSPO}}.
\end{equation}

\paragraph{KL-Regularized Update.}
Given an arm-level score vector $A_t\in\mathbb{R}^K$, we update the policy by
\begin{equation}
\label{eq:grpo-objective}
\pi_{t+1}
=
\arg\max_{\pi}
\left\{
\frac{1}{G}\sum_{i=1}^K
\frac{\pi(i)}{\pi_t(i)}A_t(i)
-
\frac{1}{\eta}
\sum_{i=1}^K
\pi(i)\log\frac{\pi(i)}{\pi_t(i)}
\right\}.
\end{equation}

In the theoretical analysis below, we use the population scores $A_t^{\text{GRPO}}$ and $A_t^{\text{VSPO}}$, obtained by replacing the empirical quantities in \Cref{eq:empirical-rs-score,eq:empirical-lsd-score} with their population counterparts. 
The explicit population-score formulas and derivations are provided in Appendices \ref{sec:population-rs-advantage} and \ref{sec:population-lsd-advantage}.

\begin{proposition}[Iteration complexity of GRPO]
\label{prop:iter-rs}
Under the GRPO population score $A_t^\mathrm{GRPO}$ in Appendix \ref{sec:population-rs-advantage}, 
we have $J(\pi_{T})\ge r^\star-\varepsilon$ when
\begin{equation}
\label{eq:rs-iters}
T\geq T^\mathrm{GRPO}(\varepsilon)
:=
\frac{\Delta_{\max}}{2\eta \left(1-\frac1G\right)\Delta_{\min}}
\log\frac{C_0}{\varepsilon}.
\end{equation}
\end{proposition}

\begin{proposition}[Iteration complexity of VSPO]
\label{prop:iter-lsd}
Define distributional reward gaps
\begin{equation*}
\label{eq:delta-def}
    \delta_{x,t}
    :=
    \mathbb{E}_{i\sim\pi_t^+}[x(i)]
    -
    \mathbb{E}_{i\sim\pi_t^-}[x(i)],
    \;
    \delta_{y,t}
    :=
    \mathbb{E}_{i\sim\pi_t^+}[y(i)]
    -
    \mathbb{E}_{i\sim\pi_t^-}[y(i)], \;
    \delta_t
    :=
    \delta_{x,t} + \alpha \delta_{y,t},
\end{equation*}
distributional gap $d_t(i):= \frac{\pi_t^+(i) - \pi_t^-(i)}{2}$, and relative distributional contrast $\rho_t(i):=\frac{d_t(i)}{\pi_t(i)}$. 
Assume the vector steering is $\gamma$-good: for all $t\ge 0$ and all
$i\neq i^\star$,
\begin{equation}
\label{eq:gamma-good}
    \frac{\delta_t}{2G}
    \left(
    \rho_t(i^\star)-\rho_t(i)
    \right)
    \geq \gamma,
    \;
    \delta_{y,t}
    \left(
    \rho_t(i^\star)-\rho_t(i)
    \right)
    \geq
    2\left(y(i^\star)-y(i)\right).
\end{equation}
Then, under the VSPO population score $A_t^\mathrm{VSPO}$ in Appendix \ref{sec:population-lsd-advantage}, 
we have $J(\pi_{T})\ge r(i^\star)-\varepsilon$ when
\begin{equation}
\label{eq:lsd-iters}
T\geq T^\mathrm{VSPO}(\varepsilon)
:=
\frac{\lambda\Delta_{\max}}
{2\eta\sqrt{1-\frac1G}\left(\left(1-\frac1G\right)\Delta_{\min}+\gamma\right)}
\log\frac{C_0}{\varepsilon},
\end{equation}
where $\lambda$ is the {bandit conditioning constant} defined in \Cref{eq:def-misalign}.
\end{proposition}

\begin{corollary}
\label{cor:lsd-faster}
When $\gamma$ defined in \Cref{eq:gamma-good} satisfies $\gamma>\min\left\{\lambda,\frac{\lambda^2}{4}\right\}\Delta_{\min}$, we have $T^{\mathrm{VSPO}}(\varepsilon)<T^{\mathrm{GRPO}}(\varepsilon)$, 
where $T^{\mathrm{GRPO}}(\varepsilon),T^{\mathrm{VSPO}}(\varepsilon)$ are defined in \Cref{eq:rs-iters,eq:lsd-iters} respectively.
\end{corollary}

\looseness=-1 The $\gamma$-good condition defined in \Cref{eq:gamma-good} has a direct interpretation. 
The first inequality requires the positive distribution $\pi_t^+$ to place relatively more mass on the target arm than on any suboptimal arm, producing an additional margin $\gamma$ in the scalarized reward direction. 
The second inequality requires this distributional contrast to be aligned with the behavioral reward $y$. 
Thus, by \Cref{cor:lsd-faster} we know VSPO improves over GRPO (with reward shaping) when the positive and negative sampling distributions create a useful contrast between high-quality and low-quality regions of the response space.

On the other hand, we can show an upper bound for $\gamma$ defined in \Cref{eq:gamma-good} (See Appendix \ref{sec:gamma-bound}). Thus, in this setting, one cannot expect VSPO to achieve arbitrarily large distributional contrast. In practice, overly large distributional contrast may also introduce significant off-policy issues during training.




\section{Experiments}\label{sec:exp}
In this section, we first explain the experimental setup (\Cref{subsec:experiment_setup}), including the implementation details of \algo\ in \Cref{subsec:implement_detail}, the tasks with 4 different target behaviors in \Cref{subsubsec:task}, and baselines to compare with in \Cref{subsubsec:baselines}. Then, we present our main experimental results in \Cref{subsec:main_results}, where we show \algo\ outperforms all the baselines in 4 settings under 3 different properties. Finally, we analyze the advantages of \algo\ in \Cref{subsec:analysis}, from the perspective of on-policy, curriculum learning and exploration diversity. We provide example responses for each target behavior in \Cref{app:solution}. 
\subsection{Experimental Setup}
\label{subsec:experiment_setup}
\subsubsection{\algo\ Implementation Details}
\label{subsec:implement_detail}
\textbf{Steering Vector Construction.}
Since the base model may not reliably produce well-separated behavioral variants on its own, we construct the steering direction with the teacher-assisted pipeline summarized in Stage 1 of \Cref{fig:algo}. In our experiments, we use 200 positive-negative pairs per target behavior. After obtaining the layer-wise vector following \cite{chen2025persona}, we select the strongest layer and normalize the resulting vector, which makes the steering coefficient comparable across tasks and behaviors. The complete construction procedure, is provided in Appendix \ref{app:Vector_Construction}.

\textbf{Self-Distillation Policy Optimization via Vector Steering.}
For rollout generation, we use steering coefficients that are symmetric around zero, e.g., $\mathcal{B}=\{-0.3,-0.15,0,0.15,0.3\}$. Rather than sampling only from the target side of the steering direction, this design forms a balanced local neighborhood around the unsteered policy. Positive coefficients encourage stronger expression of the target behavior, while zero and negative coefficients provide local reference rollouts for measuring whether target-side variants truly improve the combined task and behavior reward. This structure is especially well matched to GRPO: since GRPO normalizes rewards within each rollout group, the symmetric coefficient set induces a naturally normalized local comparison centered on the current policy. As a result, the update favors steered traces with higher relative advantage, rather than blindly distilling all positively steered generations. 
\subsubsection{Tasks and Settings} 
\label{subsubsec:task}
We evaluate our method on two complementary benchmarks: MMLU-Pro~\cite{wang2024mmlu} and MATH~\cite{hendrycksmath2021}. MMLU-Pro covers a mixture of knowledge-intensive and reasoning tasks in a broad range, while MATH focuses on structured mathematical reasoning. This combination allows us to study controllability across both general-domain QA and step-by-step problem solving. Unless otherwise specified, all experiments are conducted with Qwen3-4B.


We consider three orthogonal but practically entangled control dimensions: preference-aligned reasoning behavior, robustness to misleading context, and concise reasoning. Each dimension captures a distinct aspect of output quality, yet none can be reliably reduced to another (e.g., conciseness alone cannot explain writing type and factual reliability). 

\begin{itemize}[nosep,leftmargin=*]
\item \looseness=-1\textbf{Preference-Aligned Reasoning Behavior: Expertise Level.}
We study controllability over the level of explanation, focusing on two levels. \emph{Elementary-level} solutions are designed for readers with limited knowledge: they use simple language, short sentences, and explicit step-by-step reasoning, carefully explaining each operation while avoiding advanced notations. 
In contrast, \emph{expert-level} solutions target sophisticated readers: they use professional language, formal notation, and compact derivations, omitting routine intermediate steps and avoiding explanations of basic concepts. 




\item \textbf{Preference-Aligned Reasoning Behavior: Confidence Expression.}
We control the stylistic presentation of confidence according to user preference. The \emph{confident response} is written in a direct and decisive tone, using assertive language and minimal hedging. The \emph{cautious response} is written in a calibrated tone, using uncertainty-aware phrasing such as ``the best-supported answer is'' or ``based on the information given'', while still committing to a final option.

    \item \textbf{Robustness to Misleading Context.} We evaluate whether models can maintain correct reasoning under explicitly misleading user-provided context. We use a diverse set of strong prompts to induce agreement pressure. Models are evaluated on their ability to produce the ground-truth answer despite this bias (independence) versus following the injected answer (agreement). This setting isolates the tension between correctness and context alignment, providing a controlled testbed for studying sycophancy and its mitigation.

    \item \textbf{Concise Reasoning:} Modern LLMs, particularly distilled small models, face the verbosity issue. Excessively long responses increase latency and cost, while overly short responses may omit critical reasoning steps. Our goal is to compress response length without sacrificing correctness. A central challenge is that the optimal length is instance-dependent: different questions require different levels of detail, and there is no global target length that is uniformly optimal. 
\end{itemize}

We evaluate Preference-Aligned Reasoning Behavior using an LLM-based evaluator (Claude Sonnet 4.6~\cite{anthropic2026sonnet46}), with prompts provided in~\Cref{lst:prompt-expertise-eval}. More details are given in Appendix~\ref{app:review_level_setting}. \Cref{app:evaluator} further verify that our preference-aligned behavior results are robust to the choice of LLM evaluator.
\vspace{-5pt}
\subsubsection{Baselines}
\label{subsubsec:baselines}

We compare \algo\ with representative baselines. Since some methods, such as SDFT and SDPO, were not designed for controllable model behavior, we adapt their data selection, prompting, or reward signals to include the target behavioral attribute. Additional details are provided in \Cref{app:baseline}.

\begin{itemize}[nosep,leftmargin=*]
\item \textbf{Teacher-Trace Distillation (SFT).}
Teacher-Trace Distillation (SFT) uses an external teacher model to generate reasoning traces aligned with target behaviors, such as concise reasoning or expert-level explanations. The student model is then trained on these traces using SFT. This follows the standard offline teacher-distillation paradigm~\cite{muennighoff2025s1} and is closely related to short-CoT distillation methods~\cite{hassid2025don,kang2025c3ot}.

\item \textbf{Teacher-Trace Distillation (SFT) + GRPO.}
We further evaluate a two-stage variant that first performs teacher-trace SFT and then applies GRPO. This follows the standard post-training recipe in which SFT provides an initial policy aligned with the desired response format or behavior, and RL further optimizes the model for task reward~\cite{guo2025deepseek,zhang2025making}.
\item \looseness=-1\textbf{Self-Distillation Fine-Tuning (SDFT).} 
SDFT~\cite{shenfeld2026self} serves as a self-distillation baseline. SDFT uses the model itself as the teacher by generating training data under demonstration-conditioned prompting. In the experiments, we condition the demonstrations and prompts on the target behaviors.
    
    \item \textbf{Self-Distillation Policy Optimization (SDPO).} 
    SDPO~\cite{hubotter2026reinforcement} improves a policy using high-quality rollouts sampled from the model itself. We adapt SDPO by selecting preferred rollouts according to both task correctness and the target behavioral attribute, and then optimize the model with a GRPO-style objective. This gives SDPO access to the same behavioral criterion used in our evaluation.


\item \textbf{Reward-Shaped GRPO.}
    Reward shaping directly incorporates the target behavior into the GRPO reward. For concise reasoning, we add a length penalty to encourage shorter responses~\cite{team2025kimi,aggarwall1,zhang2025making}. For preference-aligned reasoning behaviors, we use an LLM-based evaluator to reward responses that match the desired preference. For robustness to misleading context, we do not add an extra behavior-specific reward, since the target behavior is already captured by task accuracy.

\item \textbf{GRPO with Off-Policy Textual Guidance.}
    We include a text-guided GRPO baseline inspired by methods that use auxiliary textual supervision, such as instructions or feedback, during RL training~\cite{yao2026incorporating,yan2025learning,snell2022learning,yang2026towards}. During rollout generation, we append the target-behavior instruction to the original prompt. During optimization, we remove this instruction and compute policy log probabilities under the original prompt, updating the unprompted policy with the standard GRPO objective~\cite{shao2024deepseekmath}. 

\end{itemize}
\subsection{Main Results}
\label{subsec:main_results}
\begin{wrapfigure}{r}{0.53\textwidth}
    \centering
\includegraphics[width=0.5\textwidth]{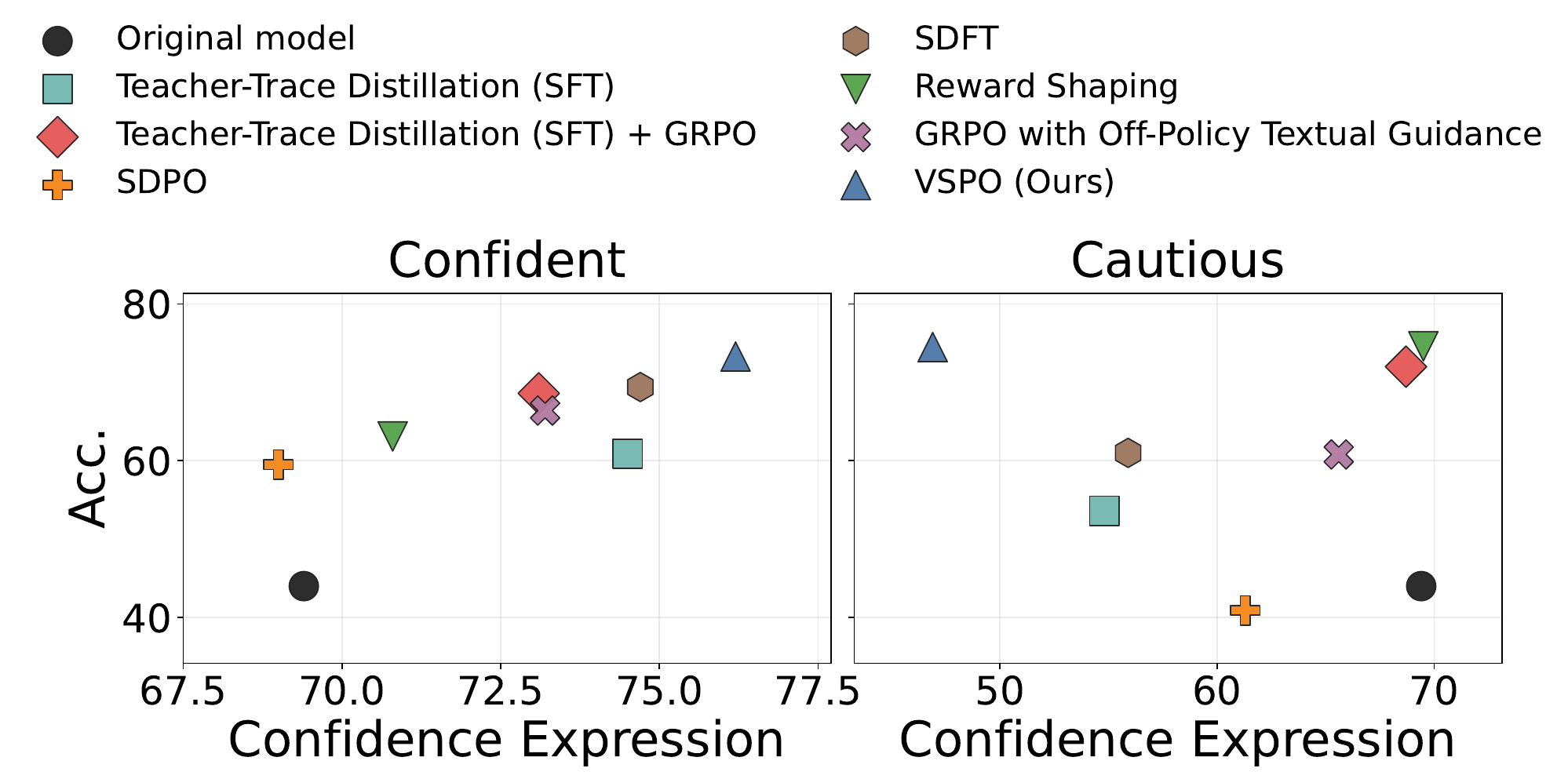}
\caption{Results on MMLU-Pro for confident and cautious target behaviors. Higher confidence-expression scores indicate a more confident style; thus, upper-right is preferred for confident control, while upper-left is preferred for cautious control.}\label{fig:confidence}
\end{wrapfigure}
\textbf{Results on Preference-Aligned Reasoning Behavior.}
We evaluate two preference-aligned reasoning characteristics: expertise level and confidence expression. As shown in \Cref{tab:review}, existing baselines exhibit a clear trade-off between behavioral control and task accuracy. Some methods can shift the response style but degrade accuracy, such as Teacher-Trace Distillation and SDFT, while others preserve accuracy but fail to induce the desired behavior reliably, such as Teacher-Trace Distillation (SFT) + GRPO and Reward-Shaped GRPO. In contrast, \algo\ achieves the strongest overall performance, steering responses toward the target expertise direction while preserving, and in some cases improving, task accuracy. We observe a similar trend for confidence expression in \Cref{fig:confidence}.


\begin{table*}[t]
\centering
\scriptsize
\resizebox{0.9\textwidth}{!}{%
\begin{tabular}{lcccc|cccc}
\hline
\multirow{3}{*}{Method} 
& \multicolumn{4}{c|}{MMLU-Pro} 
& \multicolumn{4}{c}{MATH} \\
\cline{2-9}
& \multicolumn{2}{c}{Expert} 
& \multicolumn{2}{c|}{Elementary}
& \multicolumn{2}{c}{Expert} 
& \multicolumn{2}{c}{Elementary} \\
\cline{2-9}
& Level $\uparrow$ & Acc & Level $\downarrow$ & Acc 
& Level $\uparrow$ & Acc & Level $\downarrow$ & Acc \\
\hline
Original model & 26.3 & 44.0 & 26.3 & 44.0 & 26.5 & 75.1 & 26.5 & 75.1 \\
Teacher-Trace Distillation (SFT) & 45.0 & 66.3 & 21.4 & 64.9& 43.1 & 74.0 & 22.4 & 73.9 \\
Teacher-Trace Distillation (SFT) + GRPO & 29.1 & 71.6 & 28.0 & 67.5 & 38.1 & 75.2 & 24.8 & 77.0 \\
SDFT & 33.4 & 61.3 & 20.7 & 59.3 & 40.6 & 76.3 & 23.1 & 76.9 \\
SDPO & 31.2 & 45.2 & 22.5 & 54.8 & 29.2 & 72.9 & 25.4 & 73.8 \\
Reward-Shaped GRPO & 30.6 & 63.5 & 25.8 & 60.1 & 27.0 & 75.2 & 24.3 & 73.7 \\
GRPO with Off-Policy Textual Guidance. & 30.5 & 59.2 & 26.1 & 52.9 & 39.8 & 75.3 & 20.1 & 76.4 \\
\algo\ (Ours) & \textbf{64.3} & \textbf{75.4} & \textbf{12.4} & \textbf{73.7} & \textbf{46.1} & \textbf{83.6} & \textbf{19.9} & \textbf{84.5} \\
\hline
\end{tabular}%
}
\caption{Preference-Aligned Reasoning Behavior: Expertise Level. 
We evaluate whether each method can adapt the explanation style to different explanation expertise preferences, on MMLU-Pro and MATH. ``Level'' measures explanation expertise, where higher values indicate more expert-like reasoning and lower values indicate more elementary reasoning. 
``Acc'' denotes answer accuracy. 
VSPO achieves the strongest bidirectional expertise control, while maintaining high task accuracy on both datasets.}
\label{tab:review}
\vspace{-5pt}
\end{table*}

\textbf{Results on Robustness to Misleading Context.} 
\Cref{fig:misleading} evaluates robustness to misleading context, where the desired behavior, reasoning independently rather than sycophancy, is directly measured by task accuracy. This setting should favor vanilla GRPO, since no separate behavior reward is needed: rollouts that resist the misleading context receive higher task reward. In practice, however, the original model is strongly biased by the misleading statement, achieving only 23.6\% accuracy. As a result, many rollout groups contain no correct responses, leaving GRPO with sparse signal for escaping sycophancy.
\algo\ addresses this bottleneck by using vector steering to elicit more independent reasoning traces during rollout generation. These steered rollouts increase the probability that each group contains a correct response, yielding more informative relative advantages. As shown in \Cref{fig:misleading}, \algo\ achieves the highest final accuracy 64.2\% and lower agreement rate. Its training curve starts from a higher critic score and improves faster, suggesting that latent steering exposes useful high-reward trajectories early and enables more stable optimization under misleading-context perturbations.

\begin{figure}
\centering
\includegraphics[width=\textwidth]{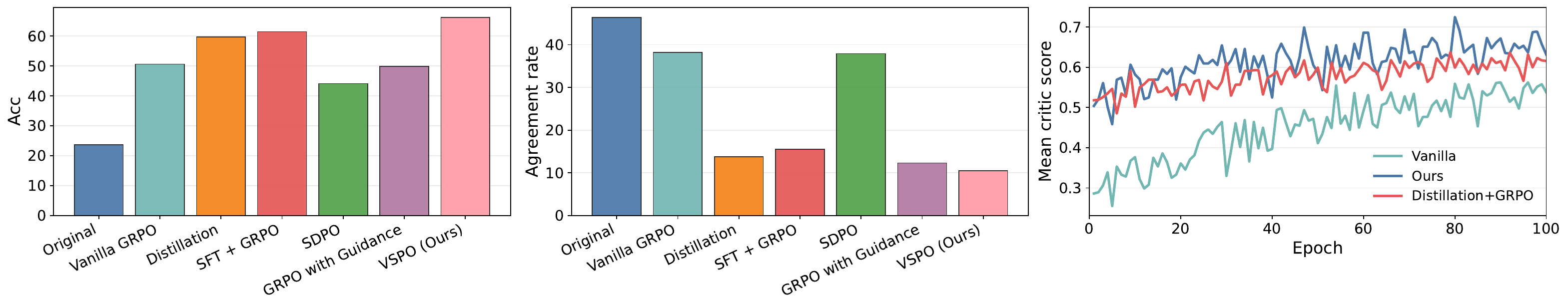}
\caption{
Results on robustness to misleading context. Left: task accuracy. Middle: agreement rate with the misleading context. Right: training dynamics measured by mean critic score. \algo\ uses vector-steered rollouts to expose more independent reasoning traces, leading to higher final accuracy, lower agreement with the misleading context, and faster improvement in critic score.
} \label{fig:misleading}
\end{figure}





\textbf{Results on Concise Reasoning.}
We evaluate concise reasoning with the goal of reducing verbosity without sacrificing accuracy. Following prior work, we use DeepSeek-R1-Distill-Qwen-7B and DeepSeek-R1-Distill-Llama-8B, and compare against length-penalty reward shaping as the main baseline. \Cref{fig:verbosity} reports the accuracy--length trade-off. Using the verbosity vectors from \cite{azizi2025activation}, \algo\ consistently improves over both vanilla GRPO and reward shaping. Notably, on MMLU-Pro, these vectors remain effective despite being constructed from mathematical questions: \algo\ achieves 60.8\% accuracy with 2345.26 tokens, compared to 59.1\% accuracy and 2616.70 tokens for reward shaping. With a Qwen3-4B verbosity vector constructed directly on MMLU-Pro, \algo\ again obtains the best trade-off. Unlike length penalties, which only adjust rewards after generation, \algo\ changes the rollout distribution by steering responses toward different verbosity levels. This structured exploration increases the chance of sampling concise yet correct traces, allowing GRPO to assign more informative advantages and distill shorter high-reward responses into the unsteered policy.

\begin{figure}
\centering
\includegraphics[width=\textwidth]{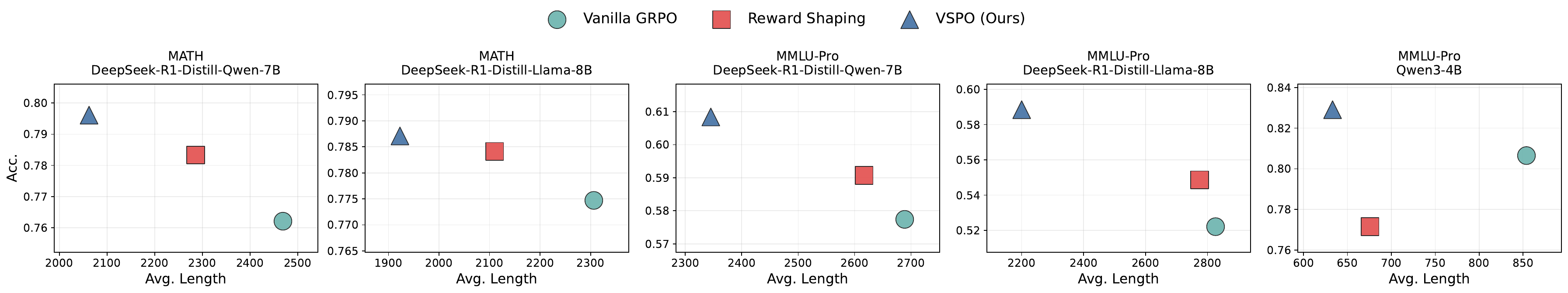}
\caption{
Accuracy-length trade-off for concise reasoning target behavior, on MATH and MMLU-Pro for different base models. VSPO consistently moves the model toward the upper-left region, achieving higher accuracy with shorter responses.
} \label{fig:verbosity}
\vspace{-5pt}
\end{figure}

\subsection{Understanding VSPO: Policy-Proximal Exploration, Curriculum Effects and Efficiency}
\label{subsec:analysis}
\begin{figure}
   \vspace{-8pt}
    \centering
    \includegraphics[width=\linewidth]{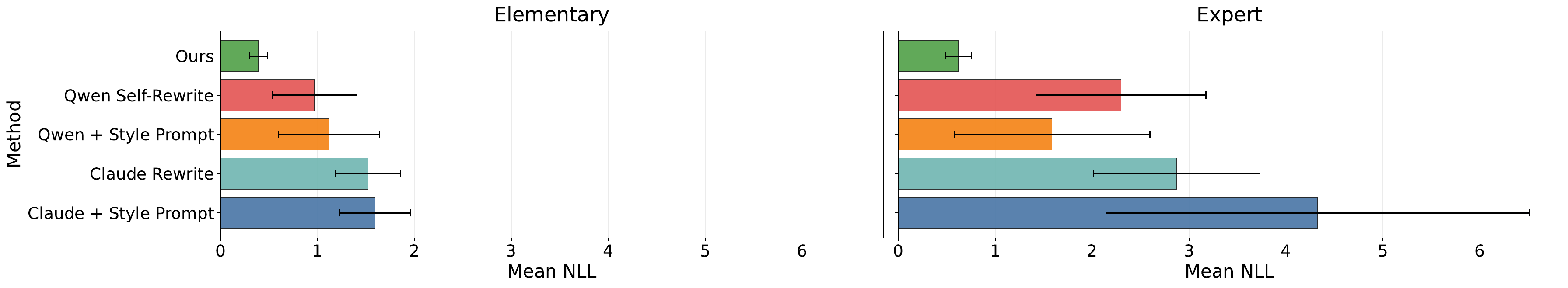}
    \vspace{-8pt}
    \caption{On-policy NLL by trace source. We report the mean token-level negative log-likelihood of generated traces under the original model. VSPO achieves substantially lower NLL than all baselines, indicating that its steered traces remain closest to the model's native distribution.
    }
        \vspace{-8pt}
    \label{fig:nll}
\end{figure}
\textbf{Vector-steered rollouts remain policy-proximal.}
We measure distributional proximity to the unsteered policy using token-level negative log-likelihood, where lower values indicate traces closer to the model's native distribution. As shown in \Cref{fig:nll}, vector-steered Qwen has substantially lower NLL than baselines in both settings, indicating much smaller distribution mismatch. This gap is especially obvious in the expert setting: because expert-style reasoning is farther from the original model's behavior, textual prompting and rewriting produce traces with very high NLL, while vector steering remains close to the unsteered policy. Notably, although VSPO samples rollouts from a range of steering coefficients during training, this NLL analysis uses only the largest coefficient in that range. Averaging over the smaller coefficients used in training would further reduce the mean NLL. These results suggest that vector steering exposes policy-proximal behavioral variants rather than imposing off-policy teacher traces, supporting stable and efficient policy optimization.

\begin{figure}
\centering
\includegraphics[width=0.9\textwidth]{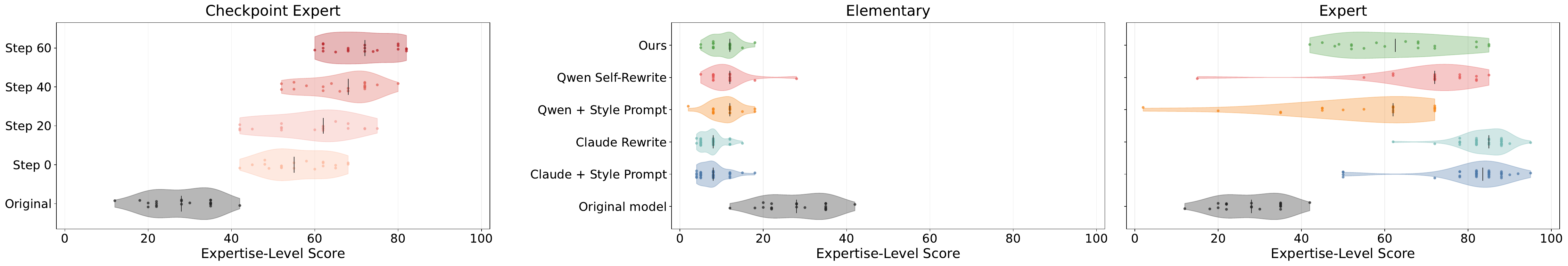}
\caption{Training dynamics and behavior-space coverage induced by vector steering.
Left: expertise-level score distributions obtained by applying the same fixed expert-style steering vector to model checkpoints at different VSPO training stages. The shift toward higher scores shows that the vector remains effective with training. Middle and right: expertise-level score distributions for different trace sources in the elementary and expert settings. VSPO produces structured shifts toward the target direction. By contrast, other baselines often provide unreliable control, either collapsing to narrow regions or producing noisy/misaligned samples.
} \label{fig:review_combine}
\vspace{-5pt}
\end{figure}


\textbf{Vector steering remains effective throughout training and induces an automatic curriculum.}
\Cref{fig:review_combine}(left) applies the same fixed expert-style steering vector to checkpoints from different stages of VSPO training. Despite changes in model weights, the vector continues to move generations in the target direction, with expertise-level scores increasing steadily over training. This indicates that the learned direction remains aligned with the evolving policy.
This persistence induces an implicit curriculum. Early in training, steering exposes moderate, policy-close behavioral shifts; later, the same vector elicits stronger and more target-aligned traces. Consequently, VSPO obtains progressively more informative self-distillation signals without recomputing the vector or relying on additional external guidance, which helps stabilize optimization and improve sample efficiency.


\textbf{Vector steering enables structured behavior exploration for more informative policy updates.}
\Cref{fig:review_combine}(middle and right) compares the expertise-level score distributions induced by different trace sources. The original model occupies a narrow style region, suggesting that ordinary sampling rarely exposes target-style traces and therefore gives reward shaping limited signal. Prompting and rewriting can shift the distribution, but the control is often unstable: baselines either collapse to narrow ranges or produce noisy, misaligned samples; in the expert setting, Qwen prompting and self-rewriting can even move responses toward more elementary styles than the original model. In contrast, vector-steered rollouts induce structured variation along the intended expertise direction: lower scores for elementary control and higher scores for expert control. For expert control, VSPO covers a broader high-score region with smoother variation because it generates rollouts with multiple steering coefficients, producing different behavior intensities along the same latent direction. This structured diversity is important for policy optimization: rather than relying on rare samples or externally rewritten traces, VSPO exposes a family of student-generated behavioral variants with different target-style intensities. These variants provide richer and more learnable supervision for distilling the desired behavior into the unsteered policy.

We provide additional \textbf{efficiency} discussion in \Cref{app:Efficiency}.

\vspace{-8pt}
\section{Discussion}
\label{sec:discussion}
\vspace{-8pt}
We introduced VSPO as a method for learning controllable model behaviors while preserving model accuracy. By using steering vectors to generate on-policy rollouts with varying intensities, VSPO alleviates sparse behavioral rewards and provides more informative candidates for GRPO-style optimization. Across explanation expertise, confidence expression, robustness to misleading context, and verbosity control, VSPO improves target control while maintaining or improving accuracy compared with reward shaping, teacher-guided distillation, and self-distillation baselines. A key limitation is that VSPO depends on the quality and alignment of the learned steering vector: if the vector is weak or fails to expose useful high-reward variants, the theoretical and empirical benefits may diminish. Future works include more robust and fully self-guided vector construction and extensions to broader multi-objective alignment settings.

\bibliography{ref}
\bibliographystyle{plain}







\appendix
\crefalias{section}{appendix}

\clearpage
\section*{Appendix}
The appendix is organized as follows:
\begin{itemize}[leftmargin=*, topsep=2pt, itemsep=1pt, parsep=0pt, partopsep=0pt]
    \item \Cref{app:llm} describes our usage of LLMs.
    \item \Cref{app:code} provides the pseudocode for VSPO.
    \item \Cref{app:related} gives additional related work.
    \item \Cref{appendix:method} presents further methodological details.
    \item \Cref{sec:proofs} contains the full theoretical analysis and proofs.
    \item \Cref{app:baseline} describes the baseline implementations.
    \item \Cref{app:exp_details} provides experimental details, including compute resources and evaluation protocols.
    \item \Cref{app:Efficiency} discusses the efficiency of VSPO.
    \item \Cref{app:evaluator} studies robustness to the choice of LLM evaluator.
    \item \Cref{app:prompt} and \Cref{app:solution} provide the prompts and example responses used in our experiments.
\end{itemize}

\section{Usage of LLMs}
\label{app:llm}
In our experiments, we use LLMs as evaluators for preference-aligned reasoning behavior and as trace generators or rewriters in several baselines. Specifically, we use Claude Sonnet 4.6 and Qwen3-4B for style evaluation, and use Claude Sonnet 4.6 and Qwen3-4B for teacher/rewrite baselines. Task accuracy is computed using task-specific answer extraction and matching rather than LLM judgments. LLMs were also used for writing assistance.

\section{Pseudo Code}\label{app:code}
\begin{algorithm}[H]
\caption{\algo: Vector-Steered Policy Optimization}
\label{alg:popcorn}
\begin{algorithmic}[1]
\REQUIRE Initial policy $\pi_\theta$, reference policy $\pi_{\mathrm{ref}}$, stronger teacher model $\mathcal{T}$, steering intensity set $\mathcal{B}=\{\beta_1,\ldots,\beta_G\}$, reward weight $\alpha$

\STATE \textbf{Stage 1: Vector construction}
\STATE Sample a set of prompts $\{q_i\}_{i=1}^{N}$ for vector construction
\FOR{each prompt $q_i$}
    \STATE Generate a base response $o_i \sim \pi_\theta(\cdot \mid q_i)$
    \STATE Use teacher $\mathcal{T}$ to rewrite $o_i$ into two opposite directions of the target property:
    \STATE \hspace{1em} positive rewrite $o_i^{+} \leftarrow \mathcal{T}(q_i, o_i, \text{positive direction})$
    \STATE \hspace{1em} negative rewrite $o_i^{-} \leftarrow \mathcal{T}(q_i, o_i, \text{negative direction})$
\ENDFOR
\STATE Form contrastive sets $\mathcal{D}^{+}=\{(q_i,o_i^{+})\}_{i=1}^{N}$ and $\mathcal{D}^{-}=\{(q_i,o_i^{-})\}_{i=1}^{N}$
\STATE Extract activations and compute steering vector $v^\ell$ using \Cref{eq:vector_construction}
\STATE Optionally normalize $v^\ell$ and select the best layer $\ell$ on validation data

\STATE \textbf{Stage 2: GRPO with vector steering}
\FOR{each training iteration}
    \STATE Sample a batch of prompts $\{q_i\}_{i=1}^{B}$
    \FOR{each prompt $q_i$}
        \FOR{each $\beta_g \in \mathcal{B}$}
            \STATE Generate steered rollout $o_{i,k}\sim \tilde{\pi}_{\theta_{\mathrm{old}}}^{(\beta_g)}(\cdot \mid q_i)$
            \STATE \hspace{1em} by injecting $\tilde{h}_t^\ell \leftarrow h_t^\ell + \beta_g v^\ell$
            \STATE Compute reward
            \[
            r_{i,g}=r_0(q_i,o_{i,g})+\alpha\beta_g
            \]
        \ENDFOR
        \STATE Compute group-relative advantages $\{A_{i,g}\}_{g=1}^{G}$ from $\{r_{i,g}\}_{g=1}^{G}$
    \ENDFOR
    \STATE Recompute token probabilities using the unsteered current policy $\pi_\theta$
    \STATE Update $\theta$ by maximizing \Cref{eq:grpo_vs}
\ENDFOR
\end{algorithmic}
\end{algorithm}

\section{Related Work}
\label{app:related}
\textbf{Teacher-Trace Distillation and SFT.}
A common alternative to RL is to use a stronger teacher model to generate reasoning traces and then train the target model with supervised fine-tuning. This paradigm has been widely used for improving reasoning performance and controlling response format, including recent work on distilling short or compressed chain-of-thought traces to reduce inference length \citep{hassid2025don,kang2025c3ot,muennighoff2025s1}. In our setting, such methods can be adapted by asking the teacher to produce traces aligned with a desired behavior, such as expert-level explanations or concise reasoning. However, teacher-trace SFT is fundamentally off-policy: the student learns from trajectories generated by an external model rather than from its own sampled rollouts. As a result, the distilled behavior may not match the student policy's reachable trajectory distribution, and further RL is often needed to recover task performance or adapt the learned behavior to the student model \citep{guo2025deepseek,zhang2025making}.

\textbf{RL Post-Training for Reasoning.}
Reinforcement learning has become a central tool for post-training LLMs, especially for reasoning tasks where rewards can be derived from verifiable correctness. Early RLHF methods optimize policies against learned preference rewards, while recent RLVR methods use task-specific verifiers to provide more direct supervision for mathematical and code reasoning. GRPO is particularly relevant to our setting because it removes the need for a learned critic and computes relative advantages within a group of sampled responses, making it a practical optimization method for reasoning models \citep{shao2024deepseekmath,guo2025deepseek}. However, standard GRPO primarily improves task performance through scalar reward signals and ordinary rollout sampling. When the desired behavior involves both correctness and a secondary characteristic, such as conciseness, confidence expression, or resistance to sycophantic bias, vanilla rollout sampling may not expose sufficiently diverse or high-quality behavioral variants for the group-relative objective to exploit.

\textbf{Self-Distillation for LLM Post-Training.}
Recent self-distillation methods aim to reduce the off-policy mismatch of teacher-trace training by deriving learning signals from the model itself. SDFT uses the model in a demonstration-conditioned mode as a self-teacher, producing distillation targets from the model's own conditional distribution rather than from fixed external traces \citep{shenfeld2026self}. SDPO further extends this idea to policy optimization by using high-quality self-generated rollouts as implicit training targets and optimizing them with a GRPO-style objective \citep{hubotter2026reinforcement}. These methods are closely related to our motivation, since they treat the current model as a source of supervision rather than relying entirely on an external teacher. However, they are primarily designed for improving task performance, not for learning controllable secondary characteristics. In contrast, our method uses vector steering to systematically reveal behaviorally diverse variants of the current policy along a target latent direction, and then distills the useful variants back into the unsteered model through RL.

\textbf{Reward Shaping and Multi-Objective RL.}
Another natural approach is to incorporate the target characteristic directly into the reward. For example, concise reasoning can be encouraged with a length penalty, while preference-aligned reasoning behavior can be rewarded with an LLM-based evaluator \citep{team2025kimi,aggarwall1,zhang2025making}. More broadly, multi-objective alignment methods study how to optimize LLMs under multiple preference dimensions or dynamically trade off competing objectives. These approaches are simple and often compatible with existing GRPO pipelines, but they only modify the scalar feedback after rollouts have already been generated. Therefore, if the rollout group lacks high-quality or behaviorally diverse candidates, reward shaping alone may provide a weak learning signal. Our method instead shapes the rollout distribution itself: vector steering actively elicits a structured set of behavioral variants before reward assignment, giving group-relative optimization more informative candidates to compare.

\textbf{Textual Guidance and AI Feedback During RL.}
Several recent methods augment RL training with additional textual guidance, such as instructions, critiques, feedback, or AI-generated supervision, to elicit desired behaviors during rollout generation \citep{snell2022learning,yao2026incorporating,yan2025learning,yang2026towards}. Related RLAIF-style approaches use powerful models as judges or feedback providers, reducing the need for human annotation but often requiring repeated model calls during training. These methods can improve controllability, but their reliance on external textual guidance or online feedback can be expensive and may introduce a mismatch between guided rollout generation and unguided inference. Our method shares the goal of eliciting more useful rollouts during RL, but replaces repeated textual guidance with a reusable latent steering vector. This makes the intervention cheaper to apply and less dependent on surface-form prompting, while still allowing the final policy to be optimized and evaluated without additional guidance at inference time.

\textbf{Persona Vectors and Vector Steering.} Persona vectors and vector steering are most closely related to the literature on \emph{representation engineering}, which studies how semantically meaningful directions in a model’s internal activation space can be identified and manipulated to control behavior at inference time \citep{zou2023representation}. Early activation-steering methods, such as Activation Addition (ActAdd) and Contrastive Activation Addition (CAA), showed that steering vectors extracted from contrastive prompts or examples can reliably shift model outputs toward target properties without updating model weights, providing a lightweight alternative to fine-tuning for controllability and alignment \citep{turner2023steering,panickssery2023steering}. Building on this line of work, persona vectors specialize such latent directions to behavioral traits or character tendencies, such as sycophancy, hallucination, or harmfulness, and have been used both to monitor personality shifts and to intervene on them during or after training \citep{chen2025persona}. This perspective is appealing because it treats behavior as at least partially organized in representation space, making control more modular and interpretable than purely prompt-level or objective-level interventions \citep{zou2023representation,chen2025persona}. At the same time, recent work has emphasized that steering effects can be highly sensitive to vector construction, layer choice, and steering magnitude, and may be entangled across multiple attributes, which limits the robustness of simple inference-time activation editing \citep{braun2025understanding,im2025unified}. In contrast to prior work that primarily applies vector steering only at inference time, our method integrates vector steering into GRPO training, using steering-induced variation during policy optimization to shape the model’s latent behavioral distribution and thereby produce more persistent and controllable adaptation.
\section{Method}
\label{appendix:method}
We propose \algo, an on-policy self-distillation algorithm that uses latent steering for credit assignment. Our starting point is the observation that the current policy already contains diverse behavioral variants, but standard RL only observes them through ordinary sampling and scalar rewards. This creates a bottleneck: the policy may contain useful behaviors that are difficult to elicit reliably, and scalar reward alone does not determine how the policy should move toward them. Our key idea is to use vector steering to expose such behaviors in a controlled way, and then use RL to distill them back into the unsteered policy.

The key object in our method is the \emph{steered policy}
\begin{equation}
\tilde{\pi}_{\theta}^{(\beta)}(\cdot \mid x),
\end{equation}
which denotes the current policy under a latent perturbation of strength $\beta$ along a learned steering direction. For a prompt $x$, the steered policy generates a rollout that can be viewed as a behaviorally modified variant of what the current model would produce on its own. By evaluating multiple such variants within each training group, we obtain a structured set of on-policy candidates that differ along a targeted behavior axis. The role of RL is then to reinforce the variants that best improve the task objective and thereby distill this behavior back into the base policy.

Algorithm~\ref{alg:popcorn} summarizes the full pipeline. First, we construct a steering vector that captures a target behavior in latent space. Second, for each prompt, we sample a group of rollouts under different steering intensities. Third, we score these rollouts with a reward that combines the primary task signal with a steering-dependent preference term, and update the unsteered policy using a GRPO-style objective. In this way, \algo\ turns vector steering into a mechanism for self-distillation: steering reveals candidate behaviors from the current model, and policy optimization internalizes the useful ones.

\subsection{Vector Construction}
We first construct a steering vector $v^\ell$ at a transformer layer $\ell$ that captures a target behavioral direction. At a high level, we obtain this direction from contrastive examples that differ mainly in the property of interest. In our setting, these contrastive examples are produced with teacher assistance: for prompts sampled from the training distribution, we first generate responses from the current small model, and then use a stronger model(Claude Sonnet 4.6) to rewrite them toward two opposite directions of the target property. We then compute a mean activation difference between the resulting positive and negative examples to obtain
\begin{equation}
v^\ell
=
\frac{1}{N}\sum_{i=1}^{N}\phi^\ell(s_i^{+})
-
\frac{1}{N}\sum_{i=1}^{N}\phi^\ell(s_i^{-}),
\label{eq:vector_construction}
\end{equation}
where $\phi^\ell(\cdot)$ denotes an activation summary at layer $\ell$, $N$ is the number of pairs, $s_i^{+}$ is the positive sample and $s_i^{-}$ is the corresponding negative sample. We optionally normalize the vector and select the best layer on held-out data. The details of vector construction are deferred to \Cref{subsec:implement_detail}.

The purpose of this stage is not to produce supervised targets for policy learning. Rather, it identifies a latent direction along which the current policy can be systematically perturbed during training. This keeps the teacher's role limited: the teacher helps define the behavior axis, while the actual learning still happens through the model's own rollouts.

\subsection{Self-Distillation Policy Optimization via Vector Steering}

We now describe how vector steering induces self-distillation during RL. For each prompt $x$, instead of sampling an unstructured group of rollouts, we generate one rollout for each steering intensity in a fixed set
\begin{equation}
\mathcal{B} = \{\beta_1,\dots,\beta_G\}.
\end{equation}
For example, $\mathcal{B}$ may be chosen as $\{-a,-b,0,b,a\}$. Given a steering intensity $\beta_g$, we inject the steering vector into the residual stream during decoding,
\begin{equation}
\tilde{h}_t^\ell = h_t^\ell + \beta_g v^\ell,
\label{eq:steer_injection}
\end{equation}
and sample
\begin{equation}
o_k \sim \tilde{\pi}_{\theta_{\mathrm{old}}}^{(\beta_g)}(\cdot \mid q).
\end{equation}

These steered rollouts remain on-policy: they are produced by the current model itself, not by an external teacher, and differ only through a controlled latent perturbation. We therefore interpret the steered policy as a \emph{self-teacher}: it reveals how the current model behaves when pushed along a targeted latent direction. This exposes a structured set of candidate behaviors that would otherwise be hard to obtain consistently from ordinary sampling.

Each rollout is assigned reward
\begin{equation}
r_g = r_0(q,o_g) + \alpha \beta_g,
\label{eq:reward_main}
\end{equation}
where $r_0$ is the primary task reward and $\alpha$ controls the steering preference. Let $\bar{r}$ and $\sigma_r$ be the mean and standard deviation of the group of rewards $\{r_k\}_{g=1}^G$, where $G$ is the number of samples in the same GRPO group. The advantge of each rollout is defined as
\begin{equation}
A_g = \frac{r_g-\bar{r}}{\sigma_r}.
\label{eq:advantage}
\end{equation}

A key design choice is that we use steering only for sampling, not for the policy update itself. For a rollout $y_g=(y_{g,1},\dots,y_{g,T_g})$, where $T_g$ is the token number of the rollout $y_g$, we compute the standard token-level ratio under the unsteered policy,
\begin{equation}
\rho_{g,t}(\theta)
=
\frac{\pi_\theta(o_{g,t}\mid q,o_{g,<t})}
{\pi_{\theta_{\mathrm{old}}}(o_{g,t}\mid q,o_{g,<t})},
\label{eq:ratio_appendix}
\end{equation}
with no steering in either numerator or denominator. The resulting objective is
\begin{equation}
\small
\mathcal{L}_{\algo}(\theta)
=
\mathbb{E}_{x}
\left[
\frac{1}{G}\sum_{g=1}^{G}
\frac{1}{T_g}\sum_{t=1}^{T_g}
\min\!\left(
\rho_{g,t}(\theta)A_g,\;
\mathrm{clip}\!\big(\rho_{g,t}(\theta),1-\epsilon,1+\epsilon\big)A_g
\right)
\right]
-
\lambda \mathcal{D}_{\mathrm{KL}}(\pi_\theta \,\|\, \pi_{\mathrm{ref}}).
\label{eq:grpo_vs_appendix}
\end{equation}

where $\mathrm{clip}$ is the clip function, $\epsilon$ is the clip ratio, $\mathcal{D}_{\mathrm{KL}}$ is the KL divergence, and $\pi_{\mathrm{ref}}$ is the reference policy. This objective has a simple interpretation. Vector steering changes the rollout distribution so that the current policy reveals diverse but targeted behavioral variants. GRPO then selects which of these variants should be preferred and updates the unsteered policy accordingly. Thus, \algo\ performs self-distillation in latent space: the model learns from steered variants of its own behavior, rather than from fully off-policy teacher traces.

\paragraph{Discussion.}
Our method is closely aligned with the spirit of self-distillation policy optimization. The main difference is the source of the distillation signal. In SDPO, the self-teacher arises by conditioning the current model on rich feedback and re-evaluating the same rollout, yielding a denser signal for credit assignment. In our setting, the self-teacher arises by steering the current model along a learned latent direction, yielding a controlled family of on-policy behavioral variants. In both cases, the current policy provides the source of supervision, and policy optimization distills that information back into the model. The advantage of our formulation is that it does not require tokenized environment feedback; instead, it uses latent steering to generate diverse and controllable rollouts directly during sampling, while remaining compatible with standard GRPO-style training.

\newenvironment{thmbis}[1]
  {\renewcommand{\thetheorem}{\ref{#1}}%
   \addtocounter{theorem}{-1}%
   \begin{theorem}}
  {\end{theorem}}
\newenvironment{propbis}[1]
  {\renewcommand{\theproposition}{\ref{#1}}%
   \addtocounter{proposition}{-1}%
   \begin{proposition}}
  {\end{proposition}}
\newenvironment{corbis}[1]
  {\renewcommand{\thecorollary}{\ref{#1}}%
   \addtocounter{corollary}{-1}%
   \begin{corollary}}
  {\end{corollary}}

\section{More on Theoretical Analysis}
\label{sec:proofs}

\subsection{More on Theoretical Setup}
\label{sec:optimal-arm}
\begin{lemma}
    Define the scalarized reward $r(i) = x(i) + \alpha y(i)$ where $\alpha >0$. If
    \begin{equation*}
\alpha>
\max_{i:\, y(i)<y(i^\star)}
\frac{\max\{x(i)-x(i^\star),0\}}
{y(i^\star)-y(i)},
\end{equation*}
then $i^\star$ defined as
\begin{equation*}
i^\star
=
\arg\max_{i\in[K]}
\left\{
x(i): y(i)=\max_{j\in[K]} y(j)
\right\}
\end{equation*}
is the unique maximizer of the scalarized reward $r(i)$, i.e., $i^\star = \arg\max_{i\in[K]} r(i)$.
\end{lemma}
\begin{proof}
    For any arm $i\ne i^\star$, if $y(i) = y(i^\star)$, then we have $r(i^\star) > r(i)$ since $x(i^\star) > x(i)$.

    If $y(i) < y(i^\star)$ and $x(i) \leq  x(i^\star)$, then we also have $r(i^\star) = x(i^\star) + \alpha y(i^\star) > x(i) + \alpha y(i) = r(i)$.

    If $y(i) < y(i^\star)$ and $x(i)  > x(i^\star)$, then by the condition of $\alpha$ we have
    \begin{equation*}
        \alpha  > \frac{x(i) - x(i^\star)}{y(i^\star) - y(i)} \implies \alpha y(i^\star)  - \alpha y(i) > x(i) - x(i^\star),
    \end{equation*}
    which implies $r(i^\star) = x(i^\star) + \alpha y(i^\star) > x(i) + \alpha y(i) = r(i)$, as desired.
\end{proof}

\subsection{Population Advantage Score for GRPO}
\label{sec:population-rs-advantage}
\begin{lemma}[GRPO Populations]
\label{lem:rs-population}
Define
\begin{equation*}
\mu_{r,t} := \sum_{j=1}^K \pi_t(j)r(j),\;
\sigma_{r,t}^2 := \sum_{j=1}^K \pi_t(j)\bigl(r(j)-\mu_{r,t}\bigr)^2.
\end{equation*}
Then, for every arm $i$,
\begin{equation*}
\mathbb{E}\left[N_{t}(i)(r(i) - \bar r_t^{\mathrm{GRPO}})\;\middle|\;\pi_t\right]=
(G-1)\pi_t(i)\bigl(r(i)-\mu_{r,t}\bigr),\;
\mathbb{E}\left[\left(\bar \sigma_t^{\mathrm{GRPO}}\right)^2\;\middle|\;\pi_t\right] =\sigma_{r,t}^2.
\end{equation*}
\end{lemma}
\begin{proof}
Since the $G$ samples are i.i.d. from $\pi_t$, we have 
\begin{equation*}
    \E[N_{t}(i)\mid \pi_t] = G\pi_t(i).
\end{equation*}
It remains to compute $\E[N_{t}(i)\bar r_t^\text{GRPO}\mid \pi_t]$. Since
\begin{equation*}
    N_{t}(i) = \sum_{g=1}^G \1\left\{a_t^{(g)}=i\right\}, \; \bar r_t^\text{GRPO} = \frac{1}{G}\sum_{h=1}^G r\left(a_t^{(h)}\right),
\end{equation*}
we have
\begin{equation*}
    \E[N_{t}(i)\bar r_t^\text{GRPO} \mid \pi_t] =
    \frac{1}{G}\sum_{g=1}^G\sum_{h=1}^G
    \E \left[\1\left\{a_t^{(g)}=i\right\}r\left(a_t^{(h)}\right) \;\middle|\; \pi_t \right].
\end{equation*}
Since
\begin{equation*}
    \E \left[\1\left\{a_t^{(g)}=i\right\}r\left(a_t^{(h)}\right) \;\middle|\; \pi_t \right] = 
    \begin{cases}
        \pi_t(i)r(i), \quad & {g=h}, \\
        \pi_t(i)\mu_{r,t}, \quad & g\ne h,
    \end{cases}
\end{equation*}
we have
\begin{equation*}
    \E[N_{t}(i)\bar r_t^\text{GRPO}\mid \pi_t] =
    \frac{1}{G} \left(G\pi_t(i)r(i) + G(G-1)\pi_t(i)\mu_{r,t}\right) = \pi_t(i)r(i) + (G-1)\pi_t(i)\mu_{r,t}.
\end{equation*}
This implies
\begin{equation*}
    \E\left[N_{t}(i)(r(i) - \bar r_t^{\text{GRPO}})\;\middle|\; \pi_t\right] =
    \E[N_{t}(i)\mid \pi_t] r_i - \E[N_{t}(i)\bar r_t^\text{GRPO}\mid \pi_t] = (G-1)\pi_t(i)(r(i)-\mu_{r,t}),
\end{equation*}
as desired.

On the other hand, notice that
\begin{equation*}
    \left(\bar \sigma_t^{\text{GRPO}}\right)^2 = \frac{1}{G-1}\sum_{g=1}^G \left(r\left(a_t^{(g)}\right)-\bar r_t^\text{GRPO}\right)^2
     = \frac{1}{G-1}\left(\sum_{g=1}^G r\left(a_t^{(g)}\right)^2  - G \left(\frac{1}{G}\sum_{g=1}^G r\left(a_t^{(g)}\right)\right)^2\right).
\end{equation*}
By linearity of expectation and independce, We have
\begin{equation*}
    \E\left[r\left(a_t^{(g)}\right)^2 \;\middle| \;\pi_t\right] = \sigma_{r,t}^2 + \mu_{r,t}^2,
\end{equation*}
and
\begin{equation*}
    \E\left[\frac{1}{G}\sum_{g=1}^G r\left(a_t^{(g)}\right) \;\middle|\; \pi_t\right] = \mu_{r,t}, \; 
    \Var\left[\frac{1}{G}\sum_{g=1}^G r\left(a_t^{(g)}\right) \;\middle|\; \pi_t\right] = \frac{\sigma_{r,t}^2}{G}.
\end{equation*}

This implies
\begin{equation*}
\E[\left(\bar \sigma_t^{\text{GRPO}}\right)^2\mid \pi_t] =
\frac{G}{G-1}\left(\sigma_{r,t}^2 + \mu_{r,t}^2\right) -
\frac{G}{G-1}\left(\mu_{r,t}^2 + \frac{\sigma_{r,t}^2}{G} \right)  = \sigma_{r,t}^2,
\end{equation*}
as desired.
\end{proof}

\begin{remark}
    In \Cref{prop:iter-rs}, we use the GRPO population arm-level score as
\begin{equation}
\label{eq:rs-population-advantage}
A_t^{\text {GRPO}}(i) :=
\frac{(G-1)\pi_t(i)(r(i)-\mu_{r,t})}{\sigma_{r,t}}.
\end{equation}
\end{remark}

\subsection{Population Advantage Score for VSPO}
\label{sec:population-lsd-advantage}

\begin{lemma}[VSPO Populations]
Define $\mu_{x,t}^\pm:=\sum_{j=1}^K \pi_t^\pm(j)x(j);\;\mu_{x,t} := \frac{\mu_{x,t}^+ + \mu_{x,t}^-}{2}, 
\mu_{y,t} := \frac{\mu_{y,t}^+ + \mu_{y,t}^-}{2}; \mu_t^\pm := \mu_{x,t}^\pm + \alpha \mu_{y,t}^\pm; \mu_t = \frac{\mu_t^+ + \mu_t^-}{2}; d_t(i) = \frac{\pi_t^+(i) - \pi_t^-(i)}{2};  \delta_{x,t} = \mu_{x,t}^+ - \mu_{x,t}^-, \delta_{y,t} = \mu_{y,t}^+ - \mu_{y,t}^-, \delta_t = \delta_{x,t} + \alpha \delta_{y,t} = \mu_t^+ - \mu_t^-$, and 
\begin{equation*}
    v_{r,t}^2:=
\frac{1}{2}\sum_{i=1}^K \pi_t^+(i)\left(r_t^+(i)-\mu_{t}\right)^2
+
\frac{1}{2}\sum_{i=1}^K \pi_t^-(i)\left(r_t^-(i)-\mu_{t}\right)^2,\\
\end{equation*}
Then, for every arm $i$,
\begin{align*}
&\mathbb{E}\Bigl[
N_t^+(i)\bigl(r_t^+(i)-\bar r_t^\mathrm{VSPO}\bigr)
+
N_t^-(i)\bigl(r_t^-(i)-\bar r_t^\mathrm{VSPO}\bigr)
\,\Big|\,\pi_t,\pi_t^+,\pi_t^-
\Bigr]
\\
&\qquad=
(G-1)\pi_t(i)\left(x(i)-\mu_{x,t}\right)
+
\frac{d_t(i)}{2}
\left(\delta_t+\alpha(G-1)\delta_{y,t}\right),\\
& \mathbb{E}\Bigl[\bigl(\bar \sigma_t^{\mathrm{VSPO}}\bigr)^2\,\Big|\,\pi_t,\pi_t^+,\pi_t^-\Bigr]
=
v_{r,t}^2+\frac{\delta_t^2}{4(G-1)}.
\end{align*}
\end{lemma}

\begin{proof}
Let $m := G/2$. At iteration $t$, applying vector steering returns $\pi_t^+$ and $\pi_t^-$ such that $\pi_t = \frac{\pi_t^+ + \pi_t^-}{2}$.
Let 
\begin{equation*}
    S_t^+ := \sum_{g=1}^m r_+\left(a_{t,+}^{(g)}\right),\; S_t^- := \sum_{g=1}^m r_-\left(a_{t,-}^{(g)}\right).
\end{equation*}
Then
\begin{equation*}
    N_{t}(i) = N_{t}^+(i) + N_{t}^-(i), \; \bar r_t^\text{VSPO} = \frac{S_t^+ + S_t^-}{G}.
\end{equation*}
By definitions we have
\begin{equation*}
    \pi_t^\pm(i) = \pi_t(i)\pm d_t(i), \; \mu_t^\pm = \mu_t \pm \frac{\delta_t}{2}.
\end{equation*}

We first compute $\E\left[(N_{t}^+(i) + N_{t}^-(i)) \;\bar r_t^\text{VSPO} \;\middle|\; \pi_t,\pi_t^+,\pi_t^-\right]$. We have
\begin{align*}
&\quad\; \E\left[(N_{t}^+(i) + N_{t}^-(i)) \;\bar r_t^\text{VSPO} \;\middle|\; \pi_t,\pi_t^+,\pi_t^-\right] \\
&=
\frac{1}{G} \E\left[(N_{t}^+(i) + N_{t}^-(i))(S_t^+ + S_t^-) \;\middle|\;\pi_t,\pi_t^+,\pi_t^-\right] \\
&=\frac{1}{G}
\left( \E[N_{t}^+(i)S_t^+] + \E[N_{t}^-(i)S_t^-] + \E[N_{t}^+(i)S_t^-] + \E[N_{t}^-(i)S_t^+]\right).
\end{align*}
Similar to the proof of \Cref{lem:rs-population} yields
\begin{align*}
    \E[N_{t}^+(i)S_t^+] &= m\pi_t^+(i)r_t^+(i) + m(m-1)\pi_t^+(i)(\mu_{x,t}^+ + \alpha\mu_{y,t}^+) \\
    &=  m\pi_t^+(i)x(i) + m(m-1)\pi_t^+(i)\mu_{x,t}^+ + m^2\pi_t^+(i)\alpha\mu_{y,t}^+\\
    \E[N_{t}^-(i)S_t^-] &= m\pi_t^-(i)r_t^-(i) + m(m-1)\pi_t^-(i)(\mu_{x,t}^- + \alpha\mu_{y,t}^-)\\
    &= m\pi_t^-(i)x(i) + m(m-1)\pi_t^-(i)\mu_{x,t}^- + m^2\pi_t^-(i)\alpha\mu_{y,t}^-
\end{align*}
By independence we have
\begin{equation*}
    \E[N_{t}^+(i)S_t^-] = m^2 \pi_t^+(i)(\mu_{x,t}^- + \alpha\mu_{y,t}^-), \; \E[N_{t}^-(i)S_t^+] = m^2 \pi_t^-(i)(\mu_{x,t}^+ + \alpha\mu_{y,t}^+).
\end{equation*}
Combining these terms yields
\begin{align*}
\E\left[(N_{t}^+(i) + N_{t}^-(i)) \;\bar r_t^\text{VSPO} \;\middle|\; \pi_t,\pi_t^+,\pi_t^-\right]
&=
\frac{1}{G}
\Big( m(\pi_t^+(i)+\pi_t^-(i))x(i) \\
&\qquad\quad + m(m-1)\left(\pi_t^+(i)\mu_{x,t}^+ + \pi_t^-(i)\mu_{x,t}^-\right) \\
&\qquad\quad + m^2\left(\pi_t^+(i)\mu_{x,t}^- + \pi_t^-(i)\mu_{x,t}^+\right)\\
&\qquad \quad + m^2\alpha (\pi_t^+(i) + \pi_t^-(i))(\mu_{y,t}^+ + \mu_{y,t}^-)
\Big).
\end{align*}
Using the facts that
\begin{align*}
    &\pi_t^+(i)+\pi_t^-(i)=2\pi_t(i),\\
    &\pi_t^+(i)\mu_{x,t}^+ + \pi_{t}^-(i)\mu_{x,t}^- = 2\pi_t(i)\mu_{x,t} + d_{t}(i)\delta_{x,t},\\
    &\pi_t^+(i)\mu_{x,t}^- + \pi_t^-(i)\mu_{x,t}^+ = 2\pi_t(i)\mu_{x,t} - d_{t}(i)\delta_{x,t},\\
    &\mu_{y,t}^+ + \mu_{y,t}^- = 2\mu_{y,t},
\end{align*}
and $G=2m$, we obtain
\begin{equation*}
    \E\left[(N_{t}^+(i) + N_{t}^-(i)) \;\bar r_t^\text{VSPO} \;\middle|\; \pi_t,\pi_t^+,\pi_t^-\right] = 
    \pi_t(i)x(i) + (G-1)\pi_t(i)\mu_{x,t} - \frac{1}{2}d_{t}(i)\delta_{x,t} + \alpha G \pi_t(i)\mu_{y,t}.
\end{equation*}
Since
\begin{equation*}
    \E[N_{t}^+(i)\mid \pi_t,\pi_t^+,\pi_t^-] = m\pi_t^+(i),\;
    \E[N_{t}^-(i)\mid \pi_t,\pi_t^+,\pi_t^-] = m\pi_t^-(i),
\end{equation*}
we have
\begin{align*}
 \E[N_{t}^+(i)r_t^+(i) + N_{t}^-(i)r_t^-(i)] &= 
    m\pi_t^+(i)[x(i) + \alpha\mu_{y,t}^+]  +m\pi_t^-(i)[x(i) + \alpha\mu_{y,t}^-] \\
    &= G\pi_t(i)x(i) + G\alpha (\pi_t(i)\mu_{y,t} + \frac12 d_{t}(i)\delta_{y,t})
\end{align*}

we conclude
\begin{align*}
&\quad \; \E\left[N_{t}^+(i) (r_t^+(i) - \bar r_t^\text{VSPO}) + N_{t}^-(i) (r_t^-(i) - \bar r_t^\text{VSPO})\;\middle |\; \pi_t,\pi_t^+,\pi_t^-\right] \\
&= (G-1)\pi_t(i)(x(i)-\mu_{x,t})+\frac{d_{t}(i)\delta_{x,t}}{2} + \frac{\alpha Gd_{t}(i)\delta_{y,t}}{2} \\
&=  (G-1)\pi_t(i)(x(i)-\mu_{x,t})+\frac{d_{t}(i)}{2}\left(\delta_{x,t} + \alpha G \delta_{y,t}\right) \\
&= (G-1)\pi_t(i)(x(i)-\mu_{x,t})+\frac{d_{t}(i)}{2}\left[\delta_{t} + \alpha (G-1) \delta_{y,t}\right].
\end{align*}
On the other hand, notice that
\begin{equation*}
    \bigl(\bar \sigma_t^\text{VSPO}\bigr)^2
    = \frac{1}{G-1}\left(
    \sum_{g=1}^{m} r_t^+\left(a_{t,+}^{(g)}\right)^2
    + \sum_{g=1}^{m} r_t^-\left(a_{t,-}^{(g)}\right)^2
    - G(\bar r_t^\text{VSPO})^2
    \right).
\end{equation*}
Recall the reward means
\begin{equation*}
    \mu_t^+ := \sum_{j=1}^K \pi_t^+(j)r_+(j)
    = \mu_{x,t}^+ + \alpha \mu_{y,t}^+,
    \;
    \mu_t^- := \sum_{j=1}^K \pi_t^-(j)r_-(j)
    = \mu_{x,t}^- + \alpha \mu_{y,t}^-,
\end{equation*}
and
\begin{equation*}
    \mu_t := \frac{\mu_t^+ + \mu_t^-}{2}
    = \mu_{x,t} + \alpha \mu_{y,t},
    \;
    \delta_t := \mu_t^+ - \mu_t^-
    = \delta_{x,t} + \alpha \delta_{y,t}.
\end{equation*}
Define the within-group variances
\begin{align*}
    &\left(\sigma_{t}^+\right)^2 := \sum_i \pi_t^+(i)\bigl(r_t^+(i)-\mu_t^+\bigr)^2
    = \sum_i \pi_t^+(i)\bigl(x(i)-\mu_{x,t}^+\bigr)^2, \\
    &\left(\sigma_{t}^-\right)^2 := \sum_i \pi_t^-(i)\bigl(r_t^-(i)-\mu_t^-\bigr)^2
    = \sum_i \pi_t^-(i)\bigl(x(i)-\mu_{x,t}^-\bigr)^2,
\end{align*}
and also recall the pooled population variance
\begin{equation*}
    v_{r,t}^2
    := \frac12 \sum_i \pi_t^+(i)\bigl(r_t^+(i)-\mu_t\bigr)^2
     + \frac12 \sum_i \pi_t^-(i)\bigl(r_t^-(i)-\mu_t\bigr)^2.
\end{equation*}

We first relate $v_{r,t}^2$ to $\left(\sigma_{t}^+\right)^2,\left(\sigma_{t}^-\right)^2$.
By expanding around $\mu_t$, we have
\begin{align*}
v_{r,t}^2
&= \frac12 \sum_i \pi_t^+(i)\bigl(r_t^+(i)-\mu_t\bigr)^2
 + \frac12 \sum_i \pi_t^-(i)\bigl(r_t^-(i)-\mu_t\bigr)^2 \\
&= \frac12\left[\left(\sigma_{t}^+\right)^2 + (\mu_t^+ - \mu_t)^2\right]
 + \frac12\left[\left(\sigma_{t}^-\right)^2 + (\mu_t^- - \mu_t)^2\right] \\
&= \frac{\left(\sigma_{t}^+\right)^2+\left(\sigma_{t}^-\right)^2}{2} + \frac{\delta_t^2}{4},
\end{align*}
which implies
\begin{equation}
\frac{\left(\sigma_{t}^+\right)^2+\left(\sigma_{t}^-\right)^2}{2}
= v_{r,t}^2 - \frac{\delta_t^2}{4}.
\label{eq:within-variance-identity-oracle-new}
\end{equation}

Similarly, we have
\begin{equation*}
    \mu_t^2 = \left(\frac{\mu_{t}^+ + \mu_{t}^-}{2}\right)^2 = \frac{\left(\mu_{t}^+\right)^2 + \left(\mu_{t}^-\right)^2}{2} - \frac{(\mu_{t}^+ - \mu_{t}^-)^2}{4} =\frac{\left(\mu_{t}^+\right)^2 + \left(\mu_{t}^-\right)^2}{2} - \frac{\delta_t^2}{4},
\end{equation*}
which implies
\begin{equation*}
    \frac{\left(\mu_{t}^+\right)^2 + \left(\mu_{t}^-\right)^2}{2} = \mu_t^2 + \frac{\delta_t^2}{4}.
\end{equation*}

Thus we have
\begin{align*}
\frac{1}{G}\E\left[
\sum_{g=1}^{m} r_t^+\left(a_{t,+}^{(g)}\right)^2
+ \sum_{g=1}^{m} r_t^-\left(a_{t,-}^{(g)}\right)^2
\;\middle|\; \pi_t,\pi_t^+,\pi_t^-
\right]
&=
\frac12\left(\left(\sigma_{t}^+\right)^2 + (\mu_t^+)^2\right)
+\frac12\left(\left(\sigma_{t}^-\right)^2 + (\mu_t^-)^2\right) \\
&= \frac{\left(\sigma_{t}^+\right)^2+\left(\sigma_{t}^-\right)^2}{2} + \frac{\left(\mu_{t}^+\right)^2 + \left(\mu_{t}^-\right)^2}{2}  \\
&= v_{r,t}^2 + \mu_t^2
\end{align*}

Notice that
\begin{equation*}
    \E[\bar r_t^\text{VSPO} \mid \pi_t,\pi_t^+,\pi_t^-] = \mu_t,
\end{equation*}
and by independence,
\begin{align*}
\Var(\bar r_t^\text{VSPO} \mid \pi_t,\pi_t^+,\pi_t^-)
&=
\frac{1}{G^2}\left(m\left(\sigma_t^+\right)^2 + m\left(\sigma_t^-\right)^2\right) \\
&=
\frac{\left(\sigma_{t}^+\right)^2+\left(\sigma_{t}^-\right)^2}{2G}
=
\frac{1}{G}\left(v_{r,t}^2 - \frac{\delta_t^2}{4}\right),
\end{align*}
where the last step uses \Cref{eq:within-variance-identity-oracle-new}.
Therefore,
\begin{equation*}
    \E[\left(\bar r_t^\text{VSPO}\right)^2 \mid \pi_t,\pi_t^+,\pi_t^-]
    =
    \mu_t^2 + \frac{1}{G}\left(v_{r,t}^2 - \frac{\delta_t^2}{4}\right).
\end{equation*}
Substituting into the definition of $\bar\sigma_t^2$, we obtain
\begin{align*}
\E[\bigl(\bar \sigma_t^\text{VSPO}\bigr)^2\mid \pi_t,\pi_t^+,\pi_t^-]
&=
\frac{G}{G-1}\left(v_{r,t}^2 + \mu_t^2\right)
-
\frac{G}{G-1}\left[
\mu_t^2 + \frac{1}{G}\left(v_{r,t}^2 - \frac{\delta_t^2}{4}\right)
\right] \\
&=
v_{r,t}^2 + \frac{\delta_t^2}{4(G-1)},
\end{align*}
as desired.
\end{proof}

\begin{remark}
    In \Cref{prop:iter-lsd}, we use the VSPO population arm-level score as
\begin{equation}
\label{eq:lsd-population-advantage}
A_t^{\text {VSPO}}(i) :=
\frac{\left(G-1\right)\pi_t(i)\bigl(x(i)-\mu_{x,t}\bigr)+\frac{d_t(i)}{2}
\left(\delta_t+\alpha(G-1)\delta_{y,t}\right)}{\sqrt{v_{r,t}^2+\frac{\delta_t^2}{4(G-1)}}}.
\end{equation}
\end{remark}

\subsection{Technical Lemmas}

\begin{lemma}[Soft policy update]
\label{lem:soft-update}
For any arm-level advantage score $A_t \in \mathbb{R}^K$, the unique optimizer of \Cref{eq:grpo-objective} is
\begin{equation}
\pi_{t+1}(i) = \frac{\pi_t(i)\exp\left(\frac{\eta A_t(i)}{G\pi_t(i)}\right)} {\sum_{j=1}^K \pi_t(j)\exp\left(\frac{\eta A_t(j)}{G\pi_t(j)}\right)}.
\label{eq:soft-update}
\end{equation}
\end{lemma}
\begin{proof}
The objective in \Cref{eq:grpo-objective} is 
\begin{equation*}
\pi_{t+1}
=
\arg\max_{\pi}
\left\{
\frac{1}{G}\sum_{i=1}^K
\frac{\pi(i)}{\pi_t(i)}A_t(i)
-
\frac{1}{\eta}
\sum_{i=1}^K
\pi(i)\log\frac{\pi(i)}{\pi_t(i)}
\right\},
\end{equation*}
which is strictly concave over the simplex, so it has a unique maximizer. Introduce the Lagrangian
\begin{equation*}
    \mathcal{L}(\pi,\lambda) = \sum_{i=1}^K \pi(i)\frac{A_t(i)}{G\pi_t(i)} - \frac{1}{\eta}\sum_{i=1}^K \pi(i)\log\frac{\pi(i)}{\pi_t(i)}
    + \lambda\left(\sum_{i=1}^K \pi(i)-1\right).
\end{equation*}
Taking derivatives with respect to $\pi(i)$ gives
\begin{equation*}
    0 = \frac{\partial \mathcal{L}}{\partial \pi(i)} = 
    \frac{A_t(i)}{G\pi_t(i)} - \frac{1}{\eta}\left(\log\frac{\pi(i)}{\pi_t(i)} + 1\right) + \lambda.
\end{equation*}
Rearranging,
\begin{equation*}
    \pi(i) = \pi_t(i)\exp\left( \frac{\eta A_t(i)}{G\pi_t(i)} + \eta \lambda - 1\right).
\end{equation*}
The term $\exp(\eta\lambda-1)$ is the same for every arm and is determined by the normalization constraint $\sum_i \pi(i)=1$. Substituting the normalizing constant yields \Cref{eq:soft-update}.
\end{proof}


\begin{lemma}
\label{lem:var-bound}
    For any bounded random variable $X\in[a,b]$, we have $\Var(X)\leq \frac{(b-a)^2}{4}$.
\end{lemma}
\begin{proof}
    Notice that $(X-a)(X-b)\leq 0$, which implies $X^2 \leq (a+b)X - ab$. Therefore, we have
    \begin{align*}
        \Var(X) &= \E[X^2] - \E^2[X] \leq (a+b)\E[X] - ab - \E^2[X]  \\
        &= -\left(\E[X]  - \frac{a+b}{2}\right)^2 + \frac{(b-a)^2}{4}\leq \frac{(b-a)^2}{4},
    \end{align*}
    as desired.
\end{proof}

\subsection{Proof of Proposition \ref{prop:iter-rs}}

\begin{propbis}{prop:iter-rs}[Iteration complexity of GRPO]
Under the GRPO population score $A_t^\mathrm{GRPO}$ in Appendix \ref{sec:population-rs-advantage}, 
we have $J(\pi_{T})\ge r^\star-\varepsilon$ when
\begin{equation*}
T\geq T^\mathrm{GRPO}(\varepsilon)
:=
\frac{\Delta_{\max}}{2\eta \left(1-\frac1G\right)\Delta_{\min}}
\log\frac{C_0}{\varepsilon}.
\end{equation*}
\end{propbis}

\begin{proof}
For every suboptimal arm $i\neq i^\star$, define
\begin{equation*}
    q_{i,t} := \frac{\pi_t(i)}{\pi_t(i^\star)}.
\end{equation*}
By \Cref{lem:soft-update} and \Cref{eq:rs-population-advantage}, we have
\begin{equation*}
    \frac{\pi_{t+1}(i)}{\pi_{t+1}(i^\star)} = \frac{\pi_{t}(i)}{\pi_{t}(i^\star)}\exp\left[\eta\left(\frac{A_t^\text{RS}(i)}{G\pi_t(i)} - \frac{A_t^\text{RS}(i^*)}{G\pi_t(i^\star)}\right)\right]
    = \frac{\pi_{t}(i)}{\pi_{t}(i^\star)}\exp\left[\eta\left(1-\frac1G\right)\frac{r(i) - r(i^\star)}{\sigma_{r,t}}\right],
\end{equation*}
which implies
\begin{equation*}
    q_{i,t+1} = q_{i,t} \exp\left[-\eta\left(1-\frac1G\right)\frac{\Delta_i}{\sigma_{r,t}}\right]
    \leq q_{i,t} \exp\left[-\eta\left(1-\frac1G\right)\frac{\Delta_{\min}}{\sigma_{r,t}}\right].
\end{equation*}
Since $r$ is bounded by $[r(i^\star) - \Delta_{\max}, r(i^\star)]$, by \Cref{lem:var-bound} we have $\sigma_{r,t} \leq \frac{\Delta_{\max}}{2}$. This implies 
\begin{equation*}
    q_{i,t+1}     \leq q_{i,t} \exp\left[-2\eta\left(1-\frac1G\right)\frac{\Delta_{\min}}{\Delta_{\max}}\right].
\end{equation*}
Iterating gives
\begin{equation*}
    q_{i,t} \leq q_{i,0} \exp\left(-\frac{2\eta \left(1-\frac1G\right)\Delta_{\min}}{\Delta_{\max}}t\right).
\end{equation*}
Notice that
\begin{equation*}
    r(i^\star)-J(\pi_t) = \sum_{i\neq i^\star}\pi_t(i)\Delta_i = \pi_t(i^\star) \sum_{i\neq i^\star}q_{i,t}\Delta_i
    \leq \sum_{i\neq i^\star}q_{i,t}\Delta_i, 
\end{equation*}
which implies
\begin{align*}
    r(i^\star)-J(\pi_t) &\leq \sum_{i\neq i^\star}q_{i,t}\Delta_i\leq \exp\left(-\frac{2\eta \left(1-\frac1G\right)\Delta_{\min}}{\Delta_{\max}}t\right)\sum_{i\ne i^\star} q_{i,0}\Delta_i \\
    &= C_0 \exp\left(-\frac{2\eta \left(1-\frac1G\right)\Delta_{\min}}{\Delta_{\max}}t\right).
\end{align*}
This further implies when 
\begin{equation*}
    T\geq T^\text{GRPO}(\varepsilon)
=
\frac{\Delta_{\max}}{2\eta \left(1-\frac1G\right)\Delta_{\min}}
\log\frac{C_0}{\varepsilon},
\end{equation*}
we have
\begin{equation*}
    r(i^\star)
-J(\pi_{T})\leq C_0 \exp\left(-\frac{2\eta \left(1-\frac1G\right)\Delta_{\min}}{\Delta_{\max}}T^\text{GRPO}(\eps)\right) \leq\eps,
\end{equation*}
as desired.
\end{proof}

\subsection{Proof of Proposition \ref{prop:iter-lsd}}

\begin{propbis}{prop:iter-lsd}[Iteration complexity of VSPO]
Define distributional reward gaps
\begin{equation*}
    \delta_{x,t}
    :=
    \mathbb{E}_{i\sim\pi_t^+}[x(i)]
    -
    \mathbb{E}_{i\sim\pi_t^-}[x(i)],
    \;
    \delta_{y,t}
    :=
    \mathbb{E}_{i\sim\pi_t^+}[y(i)]
    -
    \mathbb{E}_{i\sim\pi_t^-}[y(i)], \;
    \delta_t
    :=
    \delta_{x,t} + \alpha \delta_{y,t},
\end{equation*}
distributional gap $d_t(i):= \frac{\pi_t^+(i) - \pi_t^-(i)}{2}$, and relative distributional contrast $\rho_t(i):=\frac{d_t(i)}{\pi_t(i)}$. 
Assume the vector steering is $\gamma$-good: for all $t\ge 0$ and all
$i\neq i^\star$,
\begin{equation*}
    \frac{\delta_t}{2G}
    \left(
    \rho_t(i^\star)-\rho_t(i)
    \right)
    \geq \gamma,
    \;
    \delta_{y,t}
    \left(
    \rho_t(i^\star)-\rho_t(i)
    \right)
    \geq
    2\left(y(i^\star)-y(i)\right).
\end{equation*}
Then, under the VSPO population score $A_t^\mathrm{VSPO}$ in Appendix \ref{sec:population-lsd-advantage}, 
we have $J(\pi_{T})\ge r(i^\star)-\varepsilon$ when
\begin{equation*}
T\geq T^\mathrm{VSPO}(\varepsilon)
:=
\frac{\lambda\Delta_{\max}}
{2\eta\sqrt{1-\frac1G}\left(\left(1-\frac1G\right)\Delta_{\min}+\gamma\right)}
\log\frac{C_0}{\varepsilon},
\end{equation*}
where $\lambda$ is the {bandit conditioning constant} defined in \Cref{eq:def-misalign}.
\end{propbis}

\begin{proof}
For every suboptimal arm $i\neq i^\star$, define
\begin{equation*}
    q_{i,t} := \frac{\pi_t(i)}{\pi_t(i^\star)}.
\end{equation*}

By \Cref{lem:soft-update} and \Cref{eq:lsd-population-advantage}, we have
\begin{align*}
    &\quad \;\sqrt{v_{r,t}^2 + \frac{\delta_t^2}{4(G-1)}}\left(\frac{A_t^\text{VSPO}(i)}{G\pi_t(i)} - \frac{A_t^\text{VSPO}(i^*)}{G\pi_t(i^\star)}\right) \\
    &= \left(1-\frac1G\right)(x(i) - x(i^\star)) - \frac{\delta_t+\alpha(G-1)\delta_{y,t}}{2G}\left(\frac{d_t(i^\star)}{\pi_t(i^\star)} - \frac{d_t(i)}{\pi_t(i)}\right) \\
    &\leq  \left(1-\frac1G\right)(x(i) - x(i^\star)) - \gamma  - \frac{\alpha (G-1)}{2G}\cdot 2 (y(i^\star) - y(i)) \\
    & = \left(1-\frac1G\right)(r(i) - r(i^\star)) - \gamma\\
    & = -\left(1-\frac1G\right)\Delta_{\min} - \gamma.
\end{align*}
Also, since $r_t^+(i),r_t^-(i)$ is bounded by
\begin{equation*}
    \left[\min_{i\in[K]} x(i) + \alpha \min_{i\in[K]} y(i), \max_{i\in[K]} x(i) + \alpha \max_{i\in[K]} y(i)\right],
\end{equation*}
by \Cref{lem:var-bound} we have $\sigma_t \leq \frac{D_x + \alpha D_y}{2} = \frac{\lambda\Delta_{\max}}{2}$. Furthermore, since $\delta_t = \mu_t^+- \mu_t^-$, we have $|\delta_t|\leq D_x + \alpha D_y = \lambda \Delta_{\max}$. Thus we have
\begin{equation*}
    \sqrt{v_{r,t}^2 + \frac{\delta_t^2}{4(G-1)}} \leq \sqrt{\frac{(\lambda\Delta_{\max})^2}{4} + \frac{(\lambda\Delta_{\max})^2}{4(G-1)} }
     = \frac{\lambda\Delta_{\max}}{2\sqrt{1-\frac1G}}.
\end{equation*}
Combining the inequalities above yields
\begin{align*}
    \frac{A_t^\text{VSPO}(i)}{G\pi_t(i)} - \frac{A_t^\text{VSPO}(i^*)}{G\pi_t(i^\star)} \leq -\frac{\left(1-\frac1G\right)\Delta_{\min} + \gamma}{\sqrt{v_{r,t}^2 + \frac{\delta_t^2}{4(G-1)}}} \leq -\frac{2\sqrt{1-\frac1G}}{\lambda\Delta_{\max}}\left[\left(1-\frac1G\right)\Delta_{\min} + \gamma\right],
\end{align*}
which implies
\begin{align*}
    q_{i,t+1}  = \frac{\pi_{t+1}(i)}{\pi_{t+1}(i^\star)} &= \frac{\pi_{t}(i)}{\pi_{t}(i^\star)}\exp\left[\eta\left(\frac{A_t^\text{VSPO}(i)}{G\pi_t(i)} - \frac{A_t^\text{VSPO}(i^*)}{G\pi_t(i^\star)}\right)\right] \\
    &\leq q_{i,t} \exp\left\{-\frac{2\eta\sqrt{1-\frac1G}}{\lambda \Delta_{\max}}\left[\left(1-\frac1G\right)\Delta_{\min} + \gamma\right]\right\}.
\end{align*}

Iterating gives
\begin{equation*}
    q_{i,t} \leq q_{i,0} \exp\left\{-\frac{2\eta\sqrt{1-\frac1G}}{\lambda\Delta_{\max}}\left[\left(1-\frac1G\right)\Delta_{\min} + \gamma\right]t\right\}.
\end{equation*}
Notice that
\begin{equation*}
    r(i^\star)-J(\pi_t) = \sum_{i\neq i^\star}\pi_t(i)\Delta_i = \pi_t(i^\star) \sum_{i\neq i^\star}q_{i,t}\Delta_i
    \leq \sum_{i\neq i^\star}q_{i,t}\Delta_i, 
\end{equation*}
which implies
\begin{align*}
    r(i^\star)-J(\pi_t) &\leq \sum_{i\neq i^\star}q_{i,t}\Delta_i\leq \exp\left\{-\frac{2\eta\sqrt{1-\frac1G}}{\lambda\Delta_{\max}}\left[\left(1-\frac1G\right)\Delta_{\min} + \gamma\right]t\right\}\sum_{i\ne i^\star} q_{i,0}\Delta_i \\
    &= C_0 \exp\left\{-\frac{2\eta\sqrt{1-\frac1G}}{\lambda\Delta_{\max}}\left[\left(1-\frac1G\right)\Delta_{\min} + \gamma\right]t\right\}.
\end{align*}
This further implies when 
\begin{equation*}
   T\geq  T^\text{VSPO}(\varepsilon)
=
\frac{\lambda\Delta_{\max}}
{2\eta\sqrt{1-\frac1G}\left(\left(1-\frac1G\right)\Delta_{\min}+\gamma\right)}
\log\frac{C_0}{\varepsilon},
\end{equation*}
we have
\begin{equation*}
     r(i^\star)-J(\pi_{T})\leq C_0 \exp\left\{-\frac{2\eta\sqrt{1-\frac1G}}{\lambda\Delta_{\max}}\left[\left(1-\frac1G\right)\Delta_{\min} + \gamma\right]T^\text{VSPO}(\varepsilon)\right\} \leq\eps,
\end{equation*}
as desired.
\end{proof}

\subsection{Proof of Corollary \ref{cor:lsd-faster}}
\begin{corbis}{cor:lsd-faster}
When $\gamma$ defined in \Cref{eq:gamma-good} satisfies $\gamma>\min\left\{\lambda,\frac{\lambda^2}{4}\right\}\Delta_{\min}$, then $T^{\mathrm{VSPO}}(\varepsilon)<T^{\mathrm{GRPO}}(\varepsilon)$, 
where $T^{\mathrm{GRPO}}(\varepsilon),T^{\mathrm{VSPO}}(\varepsilon)$ are defined in \Cref{eq:rs-iters,eq:lsd-iters} respectively.
\end{corbis}

\begin{proof}
Notice that 
\begin{equation*}
    \left[\lambda \sqrt{1-\frac1G} - \left(1 - \frac1G\right)\right]\Delta_{\min} < \lambda \Delta_{\min}, \; \left[\lambda \sqrt{1-\frac1G} - \left(1 - \frac1G\right)\right]\Delta_{\min}\leq \frac{\lambda^2}{4} \Delta_{\min} ,
\end{equation*}
which implies
\begin{equation*}
    \lambda  > \min\left\{\lambda,\frac{\lambda^2}{4}\right\}\Delta_{\min} \geq \left[\lambda \sqrt{1-\frac1G} - \left(1 - \frac1G\right)\right]\Delta_{\min}.
\end{equation*}
Thus we have
\begin{equation*}
    T^\text{VSPO}(\varepsilon) = \frac{\lambda\Delta_{\max}}
{2\eta\sqrt{1-\frac1G}\left(\left(1-\frac1G\right)\Delta_{\min}+\gamma\right)}  < 
\frac{\lambda\Delta_{\max}}
{2\eta\sqrt{1-\frac1G}\cdot \lambda \sqrt{1-\frac1G}\Delta_{\min}} = T^\text{GRPO}(\varepsilon),
\end{equation*}
as desired.
\end{proof}

\subsection{More on $\gamma$-Good Condition}
\label{sec:gamma-bound}
\begin{lemma}
    $\gamma\leq \frac{\lambda\Delta_{\max}}{G}$ where $\gamma$ is defined in \Cref{eq:gamma-good}, and $\lambda$ is the \textit{bandit conditioning constant} defined in \Cref{eq:def-misalign}.
\end{lemma}
\begin{proof}
    Notice that for any arm $i$, we have
    \begin{equation*}
        d_t(i) = \frac{\pi_t^+(i) - \pi_t^-(i)}{2}, \; \pi_t(i) = \frac{\pi_t^+(i) + \pi_t^-(i)}{2},
    \end{equation*}
    which implies $\pi_t^+(i) = \pi_t(i) + d_t(i)$. Thus we have
    \begin{equation*}
        \pi_t^+(i) = \pi_t(i)  + d_t(i) \leq \pi_t(i) + \frac{\pi_t^+(i)}{2} \implies \pi_t^+(i)\leq 2\pi_t(i),
    \end{equation*}
    which implies
    \begin{equation*}
        \rho_t(i) = \frac{d_t(i)}{\pi_t(i)} = \frac{\pi_t^+(i) - \pi_t(i)}{\pi_t(i)} = \frac{\pi_t^+(i)}{\pi_t(i)} - 1 \in [-1,1].
    \end{equation*}
    Thus we have $-2\leq \rho_t(i^\star) - \rho_t(i)\leq 2$ for any arm $i$. 

    Since $\delta_t = \mu_t^+- \mu_t^-$, we have $|\delta_t|\leq D_x + \alpha D_y = \lambda \Delta_{\max}$. This implies
    \begin{equation*}
        \gamma \leq \frac{\delta_t}{2G}\left(\rho(i^\star) - \rho(i)\right) \leq \frac{|\delta_t|}{2G}\left|\rho(i^\star) - \rho(i)\right| \leq \frac{\lambda \Delta_{\max}}{G},
    \end{equation*}
    as desired.
\end{proof}

\section{Baseline}\label{app:baseline}

We compare \algo\ against representative approaches that can be adapted to learning controllable reasoning behaviors. These baselines cover teacher-supervised distillation, self-distillation, prompt-conditioned RL, and reward-shaped RL. Importantly, several methods, such as SDFT and SDPO, were not originally designed for controllable model behavior. They primarily aim to improve task performance through self-distillation. To make them comparable in our setting, we adapt their selection, prompting, or reward signals to include the target behavioral attribute when applicable.
\begin{itemize}[nosep,leftmargin=*]
            \item \textbf{Teacher-Trace Distillation (SFT).}
We adapt recent short-CoT distillation methods, which train models on compressed reasoning traces to reduce inference length~\cite{hassid2025don,kang2025c3ot}, to our broader controllable-behavior settings. Specifically, we use an external teacher model to generate reasoning traces aligned with each target behavior, such as expert-level explanations and concise reasoning, and then fine-tune the student model with Supervised Fine-Tuning (SFT). This baseline follows the standard offline teacher-distillation paradigm~\cite{muennighoff2025s1}. For fair comparison, we use the same number of trace as our vector construction.
\item \textbf{Teacher-Trace Distillation (SFT) + GRPO.}
We further evaluate a two-stage variant that first performs teacher-trace SFT and then applies GRPO. This follows the standard post-training recipe in which SFT provides an initial policy aligned with the desired response format or behavior, and RL further optimizes the model for task reward~\cite{guo2025deepseek,zhang2025making}. For this baseline, we use the Teacher-Trace Distillation (SFT) checkpoint as the initialization and continue training with GRPO for the same number of RL steps as VSPO, namely 100 steps.
\item \textbf{SDFT (Self-Distillation Fine-Tuning).} 
    We include Self-Distillation Fine-Tuning (SDFT)~\cite{shenfeld2026self} as a self-distillation baseline. SDFT uses the model itself as a teacher under demonstration-conditioned prompting, producing training signals from the model's own conditional distribution rather than from fixed external traces. Since SDFT was originally proposed for improving policy learning rather than controllable model behavior, we adapt it by conditioning demonstrations and target behavior prompts on the target attribute. 
    
    \item \textbf{SDPO (Self-Distillation Policy Optimization).} We compare with Self-Distillation Policy Optimization (SDPO)~\cite{hubotter2026reinforcement}, a self-distillation RL method that improves the policy using its own high-quality rollouts. The model samples multiple candidate responses and treats stronger outputs (e.g., based on correctness and our target) as implicit targets, which are then optimized via a GRPO-style objective. In its original form, SDPO is designed for policy improvement, not for controllable model behavior. We therefore adapt it to our setting by selecting stronger rollouts according to a combined criterion for both correctness and the target attribute. This gives SDPO access to the same behavioral objective used in our evaluation. We train 100 steps for fair comparison. 
\item \textbf{Reward-Shaped GRPO.}
We incorporate the target behavior directly into the GRPO reward. For concise reasoning, we add a length penalty to encourage shorter responses~\cite{team2025kimi,aggarwall1,zhang2025making}. For Preference-Aligned Reasoning Behavior, such as expertise-level control, we use an LLM-based evaluator to reward responses that match the desired user preference. Because this baseline places an LLM inside the RL loop for evaluation, it incurs substantial computational and monetary cost. For Robustness to Misleading Context, we do not introduce an additional behavior-specific reward, since the desired behavior is reflected by task accuracy. We train 100 steps for fair comparison. 
\item \textbf{GRPO with Off-Policy Textual Guidance.}
We include a text-guided GRPO baseline inspired by recent RL methods that augment standard rollouts with auxiliary textual supervision, such as instructions or feedback~\cite{yao2026incorporating,yan2025learning,snell2022learning,yang2026towards}. These methods use additional context to elicit desired behaviors during generation and then train the model to internalize those behaviors so they can be produced without the extra context at inference time. 
Concretely, we append the target-behavior instruction to the original prompt only during rollout generation. During optimization, we remove this instruction and compute policy log probabilities under the original prompt, updating the unprompted policy with the standard GRPO objective~\cite{shao2024deepseekmath}. 
This baseline tests whether our latent vector steering provides a more effective mechanism for eliciting optimization-useful target-behavior rollouts than explicit behavior-inducing prompts.

\end{itemize}

We do not adopt the finetuning-based intervention from Persona Vectors~\citep{chen2025persona} because its goal is different from ours. Persona Vectors studies how latent trait directions can monitor, predict, and prevent undesirable persona shifts during finetuning, including a preventative steering procedure designed to suppress harmful personality drift. In contrast, our goal is not to prevent an unwanted trait from emerging during standard finetuning, but to use a target behavior direction as an exploration mechanism inside RL. VSPO therefore applies vector steering only during rollout generation, evaluates the resulting behavior-accuracy trade-off with GRPO rewards, and updates the unsteered policy. The results of the finetuning-based intervention from Persona Vectors are reported in \Cref{tab:mmlu_expert_persona}. This approach performs substantially worse than the baselines considered in the main paper, both in expertise-level control and task accuracy, so we do not include it as a primary baseline.
\section{Experimental details} \label{app:exp_details}
\paragraph{Compute Resources and Reproducibility.}
We report the compute resources used for each stage of the experimental pipeline. Claude API calls for teacher-assisted rewriting and LLM-based behavioral evaluation are performed through AWS Bedrock. Steering-vector construction consists of forward passes over contrastive response pairs and is run on $4\times$A40 GPUs. Policy optimization experiments are run on GPU nodes with either $8\times$H200 GPUs or $4$--$8\times$L40S GPUs, depending on the model and training configuration. Each training run uses the same GPU type within a run. Unless otherwise specified, all inference, and downstream evaluation are performed on $4\times$A40 GPUs. Unless otherwise specified, we train 100 steps for RL and 2 epochs for SFT.

The main experiments use 4-bit or bfloat16 model inference depending on the backend configuration, and all policy optimization runs use distributed training with vLLM-based rollout generation. Wall-clock time varies with model size and rollout length.

\subsection{Vector steering}
\begin{table}[t]
\centering
\small
\setlength{\tabcolsep}{6pt}
\renewcommand{\arraystretch}{1.15}
\begin{tabular}{llcc}
\toprule
\textbf{Property} 
& \textbf{Rewrite Pair} 
& \textbf{Positive Condition} 
& \textbf{Negative Condition} \\
\midrule
Expertise Level 
& Expert / Elementary 
& $\geq 70$ 
& $\leq 30$ \\

Confidence Expression 
& Confident / Cautious 
& $\geq 60$ 
& $\leq 40$ \\

Robustness to Misleading Context 
& Robust / Sycophantic 
& Correct answer 
& Incorrect answer \\

Concise Reasoning 
& Concise / Verbose 
& $\geq 1536$
& $\geq 512$ \\
\bottomrule
\end{tabular}
\caption{
Property-specific filtering thresholds for teacher-generated contrastive rewrites. 
A rewrite pair is retained only if both rewrites satisfy the corresponding filtering conditions. 
For expertise level and confidence expression, we use a 0--100 evaluator score and retain pairs whose positive rewrite scores at least 70 and whose negative rewrite scores at most 20. 
For robustness to misleading context, the contrast is defined by answer correctness under the misleading prompt rather than by a continuous style score. 
For concise reasoning, lower verbosity indicates stronger alignment with the concise direction.
}
\label{tab:rewrite_thresholds}
\end{table}
\subsubsection{Steering Vector Construction.} \label{app:Vector_Construction}
We construct steering vectors using a teacher-assisted contrastive pipeline. For each target model and target property, we first sample reasoning traces from the target model itself (e.g., Qwen3-4B) and retain only traces that lead to the correct answer, allowing up to five sampling attempts per prompt. Each retained trace is then rewritten by Claude in two opposite directions of the target property, yielding a positive and a negative trace pair (e.g., expert vs. elementary, or confident vs. cautious). We again allow up to five rewriting attempts and filter the rewritten traces by both task correctness and a property-specific score threshold; the thresholds for each setting are reported in \Cref{tab:rewrite_thresholds} in the appendix. In total, we collect 200 positive--negative rewrite pairs for each target property.

Following \cite{chen2025persona}, we compute a candidate steering vector at each layer as the mean activation difference between the positive and negative rewritten traces, conditioned on the same original prompt:
\[
v_\ell = \frac{1}{N}\sum_{i=1}^{N}
\left(h_{\ell}^{+}(x_i, y_i^{+}) - h_{\ell}^{-}(x_i, y_i^{-})\right),
\]
where $h_{\ell}^{+}$ and $h_{\ell}^{-}$ denote the layer-$\ell$ activations induced by the positive and negative traces, respectively. We then select the layer whose vector induces the strongest separation along the target-property evaluator. To make steering coefficients comparable across target properties and tasks, we normalize the selected vector according to its empirical scoring magnitude, so that a unit steering coefficient approximately corresponds to moving the generation toward the maximum or minimum end of the observed property-score range. The resulting vector is fixed throughout RL optimization and is only used during rollout generation.

In practice, we recommend using smaller $\beta$ values since RL iterations can always push to the desired direction gradually, and note that the vector is fixed throughout optimization. 

To construct each steering vector, we compute contrastive activation differences between positive and negative behavior demonstrations at each transformer layer. Since the separability and controllability of the target behavior can vary substantially across layers, we perform a layer search by applying vectors extracted from different layers and evaluating the resulting accuracy--behavior trade-off. As shown in \Cref{fig:layer_search}, the best layer is selected separately for each target behavior and then fixed for all subsequent vector-steered rollouts.

\begin{table}[t]
\centering
\small
\setlength{\tabcolsep}{6pt}
\renewcommand{\arraystretch}{1.12}
\begin{tabular}{lcc}
\toprule
\textbf{Method} & \textbf{Expertise Level $\uparrow$} & \textbf{Accuracy $\uparrow$} \\
\midrule
Teacher-Trace Distillation (SFT) & 45.0 & 66.3 \\
Finetuning-Based Intervention from Persona Vectors & 30.7 & 44.8 \\
VSPO (Ours) & \textbf{64.3} & \textbf{75.4} \\
\bottomrule
\end{tabular}
\caption{
Comparison with the finetuning-based intervention from Persona Vectors on MMLU-Pro under the expert target behavior. VSPO achieves substantially stronger expert-style control while also improving task accuracy.
}
\label{tab:mmlu_expert_persona}
\end{table}

\begin{figure}[t]
    \centering
    \begin{subfigure}{0.32\textwidth}
        \centering
        \includegraphics[width=\linewidth]{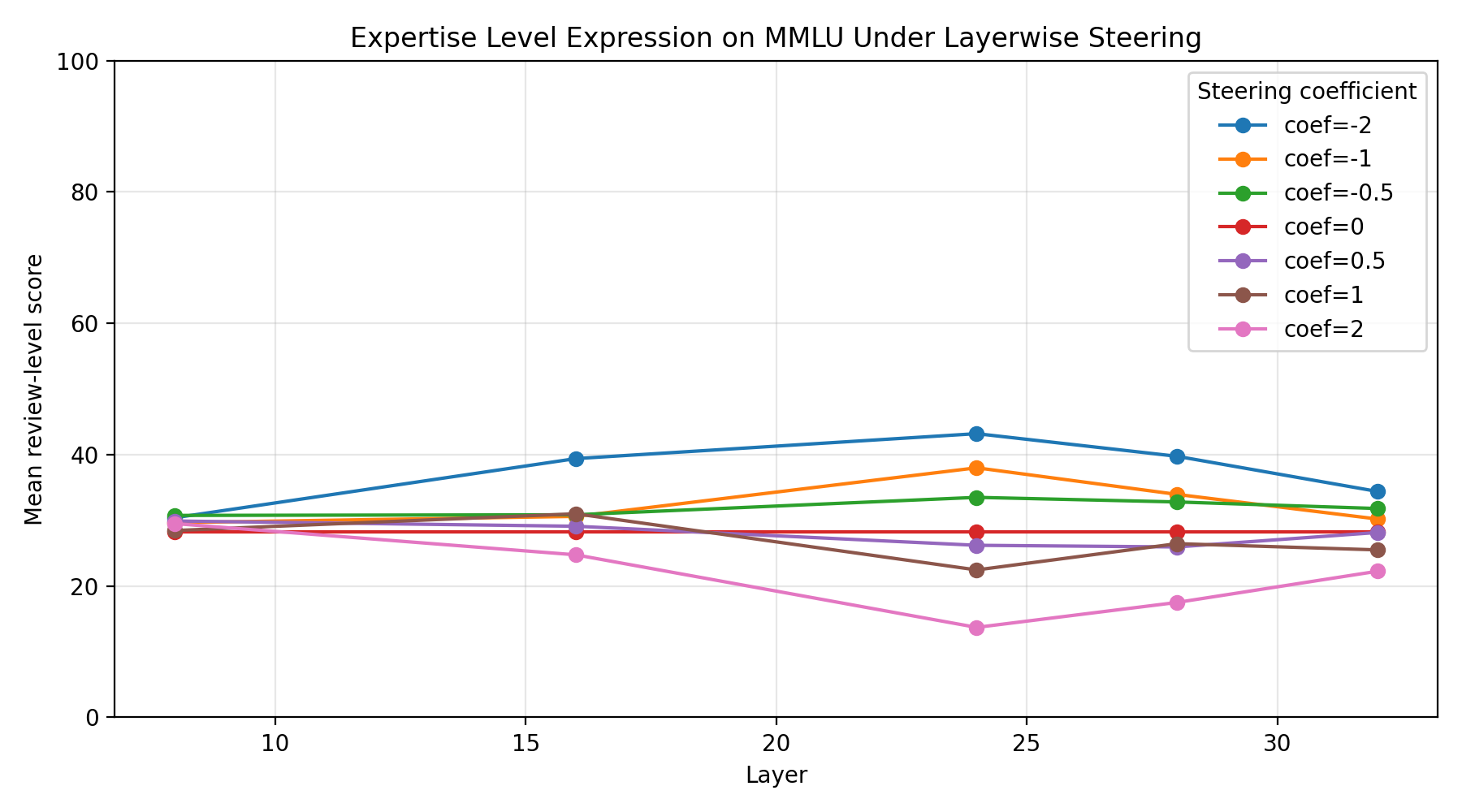}
        \caption{Expertise level}
        \label{fig:layer_search_expertise}
    \end{subfigure}
    \hfill
    \begin{subfigure}{0.32\textwidth}
        \centering
        \includegraphics[width=\linewidth]{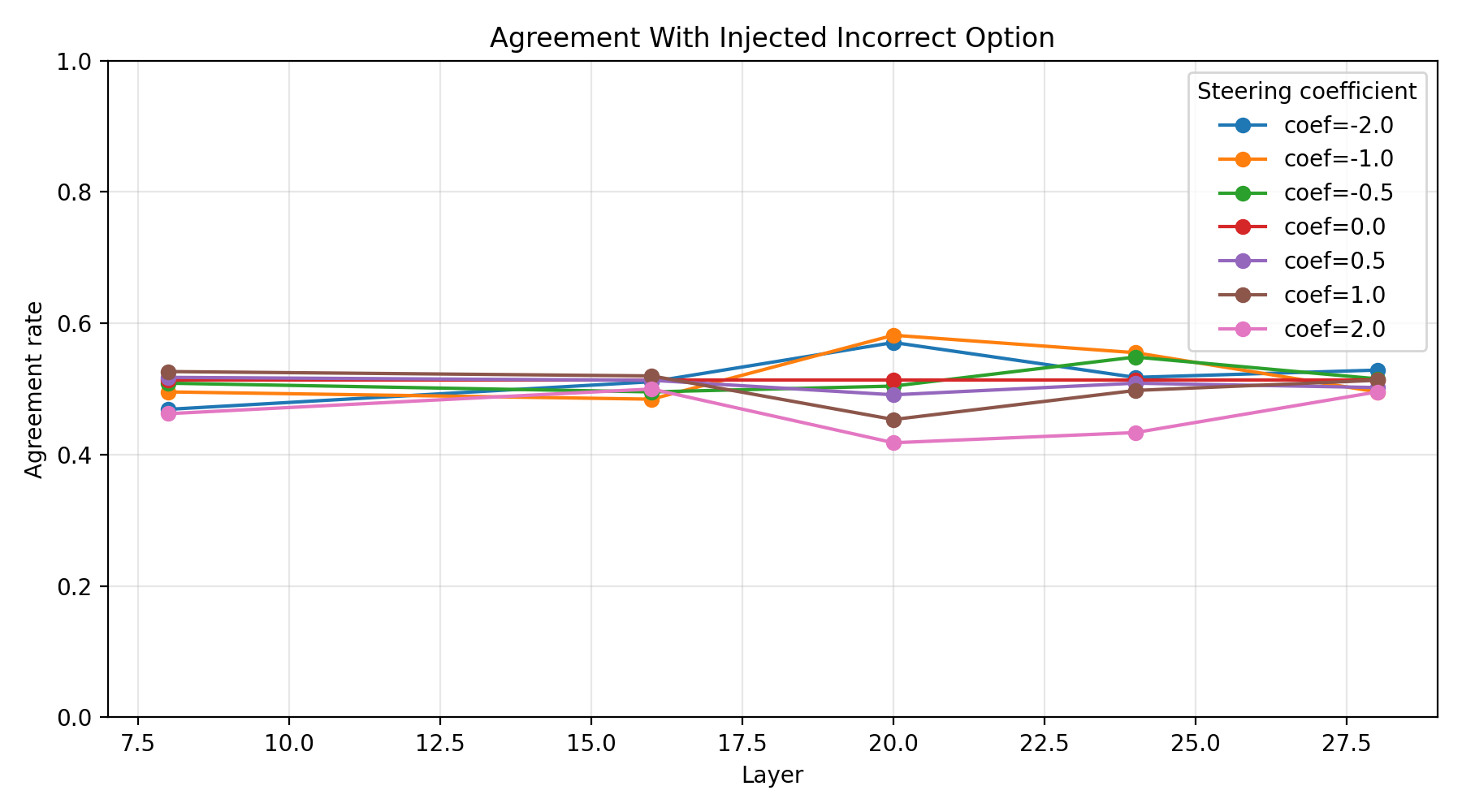}
        \caption{Misleading Context}
        \label{fig:layer_search_confidence}
    \end{subfigure}
    \hfill
    \begin{subfigure}{0.32\textwidth}
        \centering
        \includegraphics[width=\linewidth]{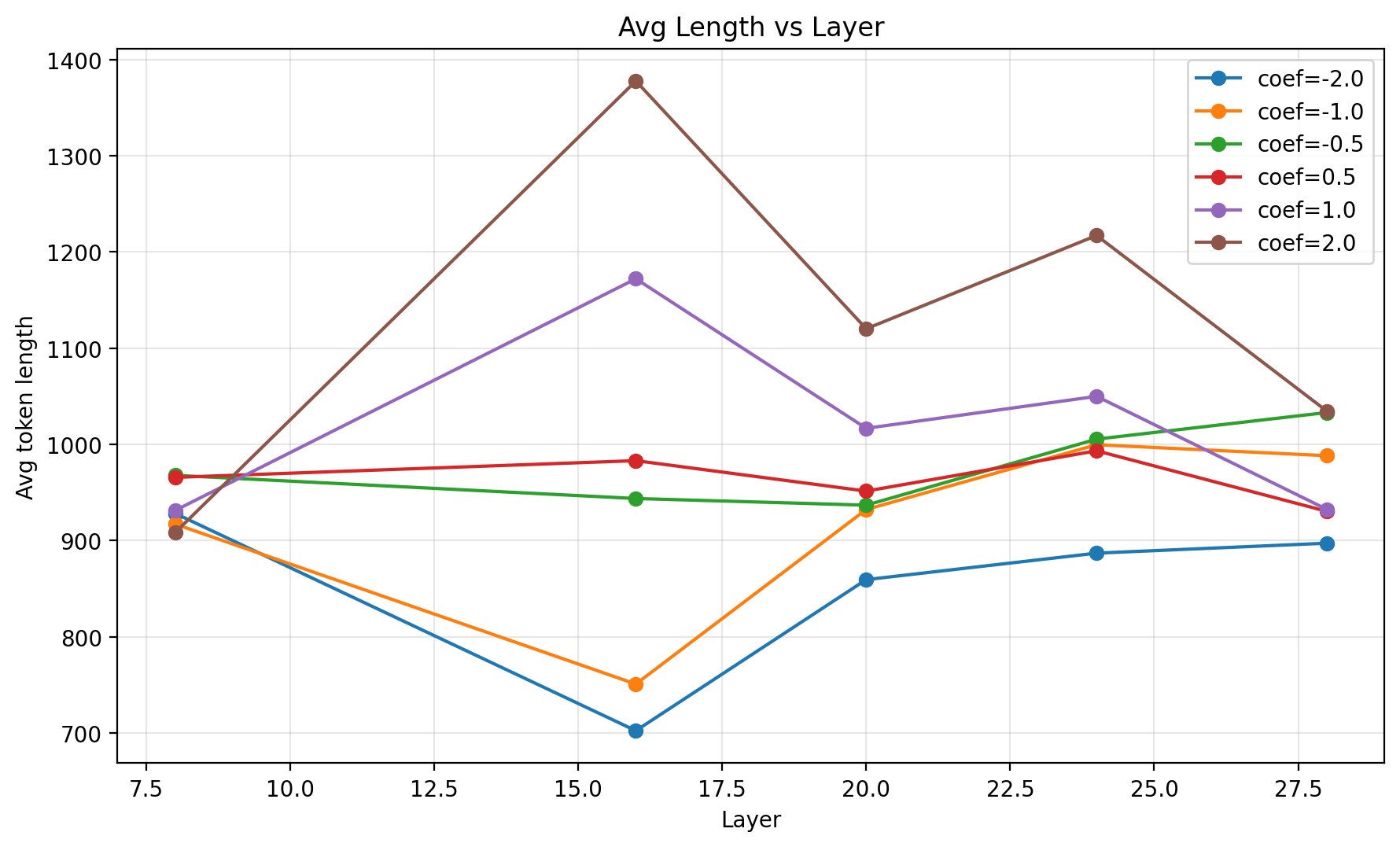}
        \caption{Concise reasoning}
        \label{fig:layer_search_verbosity}
    \end{subfigure}
    \caption{
    Layer selection for steering-vector construction. For each target behavior, we evaluate steering vectors computed from different transformer layers and select the layer that best improves the target behavioral score while preserving task accuracy. The selected layer is then fixed and used for all subsequent vector-steered rollouts in \algo.
    }
    \label{fig:layer_search}
\end{figure}

\subsubsection{Expertise level} \label{app:review_level_setting}
\textbf{Elementary-Level Explanation}: The elementary-level condition is designed to elicit solutions written for readers with minimal mathematical background. Responses in this style use simple language, short sentences, and explicit step-by-step reasoning. The model is instructed to avoid advanced terminology and dense notation whenever possible, to explain the meaning of each quantity or operation, and to present calculations slowly and transparently. The target style is pedagogical and beginner-oriented: routine steps are not skipped, symbols are introduced sparingly, and the solution is structured to make the reasoning accessible to a young or novice learner.

\textbf{Expert-Level Explanation}:
The expert-level condition is designed to elicit solutions written for mathematically sophisticated readers. Responses in this style use precise mathematical language, standard notation, and compact derivations. The model is instructed to omit routine intermediate steps, compress algebraic or arithmetic manipulations, and avoid pedagogical explanations of basic concepts. The target style is concise, formal, and verification-oriented: the solution should contain enough justification to establish correctness while assuming that the reader is comfortable with abstraction, notation, and standard mathematical conventions.

\textbf{Vector generation}: To construct the review-level direction, we form paired solutions for a shared set of problems. For each problem $x_i$, we generate two stylistically contrasting solutions: an elementary-level solution $y_i^{\mathrm{elem}}$, characterized by slow, explicit, beginner-oriented reasoning, and an expert-level solution $y_i^{\mathrm{expert}}$, characterized by concise derivations, formal mathematical language, and omission of routine steps. Our current pipeline constructs 200 such aligned elementary–expert pairs from the MMLU-Pro training split.

We consider three generation settings: (1) both solutions are generated directly by Qwen; (2) Qwen generates an initial solution, which is then rewritten into the target style by Claude Sonnet 4.6; and (3) Claude Sonnet 4.6 directly rewrites the original reference solution into the target style. Unless otherwise specified, we adopt setting (2). We further filter the outputs to retain 200 pairs in which both solutions are correct, ensuring stylistic contrast without confounding errors.

\textbf{Evaluation}: We evaluate review-level style using an LLM judge. The judge assigns each generated response a score from 0 to 100 based only on explanation style, where 0 indicates an extremely elementary, step-by-step explanation and 100 indicates compact expert-level exposition. The rubric explicitly instructs the judge to ignore correctness and focus on language complexity, step granularity, notation, abstraction level, and pedagogical tone.
Second, we measure task accuracy by extracting the predicted answer choice from the model response and comparing it to the ground-truth MMLU answer.

\begin{table*}[t]
\centering
\scriptsize
\resizebox{0.9\textwidth}{!}{%
\begin{tabular}{lcccc|cccc}
\hline
\multirow{3}{*}{Method} 
& \multicolumn{4}{c|}{MMLU-Pro} 
& \multicolumn{4}{c}{MATH} \\
\cline{2-9}
& \multicolumn{2}{c}{Expert} 
& \multicolumn{2}{c|}{Elementary}
& \multicolumn{2}{c}{Expert} 
& \multicolumn{2}{c}{Elementary} \\
\cline{2-9}
& Claude & Qwen & Claude & Qwen  & Claude & Qwen & Claude & Qwen \\
\hline
Teacher-Trace Distillation (SFT) & 45.0 & 52.9 & 21.4 & 28.3& 43.1 & 57.9 & 22.4 & 31.0 \\
\algo\ (Ours) & 64.3 & 73.6 & 12.4 & 20.5 & 46.1 & 65.8 & 19.9 & 22.7 \\
\hline
\end{tabular}%
}
\caption{Expertise level evaluation using both Claude API and Qwen3-4B.}
\label{tab:qwen_evaluator}
\end{table*}
\section{Efficiency} \label{app:Efficiency}
\textbf{Vector steering is more cost-efficient than LLM-in-the-loop baselines.}
Our method only requires a one-time external-model cost to construct a small set of supervision pairs (e.g., 200 pairs). Once the steering vector is computed, it is reused throughout training without additional large-model calls. In contrast, LLM-in-the-loop baselines that rely on feedback, rewriting, or reward evaluation require repeated external calls during rollout generation or scoring, which becomes expensive in RL settings with multiple rollouts per step. Concretely, the total external cost of our method is comparable to roughly a single RL step of rewriting or feedback (e.g., 8 rollouts with batch size 64), after which no further external supervision is needed.

\paragraph{Sample Efficiency of Steering-Vector Construction.}
We evaluate the effect of the number of contrastive pairs used to construct the steering vector. As shown in \Cref{fig:pairs_vs_expertise}, the expertise score increases rapidly with the first 20-50 pairs and then begins to saturate, indicating that most of the behavioral control benefit is obtained from a relatively small supervision set. Adding more pairs provides only limited additional improvement. This demonstrates that \algo\ is sample-efficient in its vector construction stage: it requires only a few hundred teacher-generated contrastive examples as a one-time offline cost, after which the learned direction can be reused across rollout generation and training.

\begin{figure}[t]
    \centering
    \includegraphics[width=0.55\textwidth]{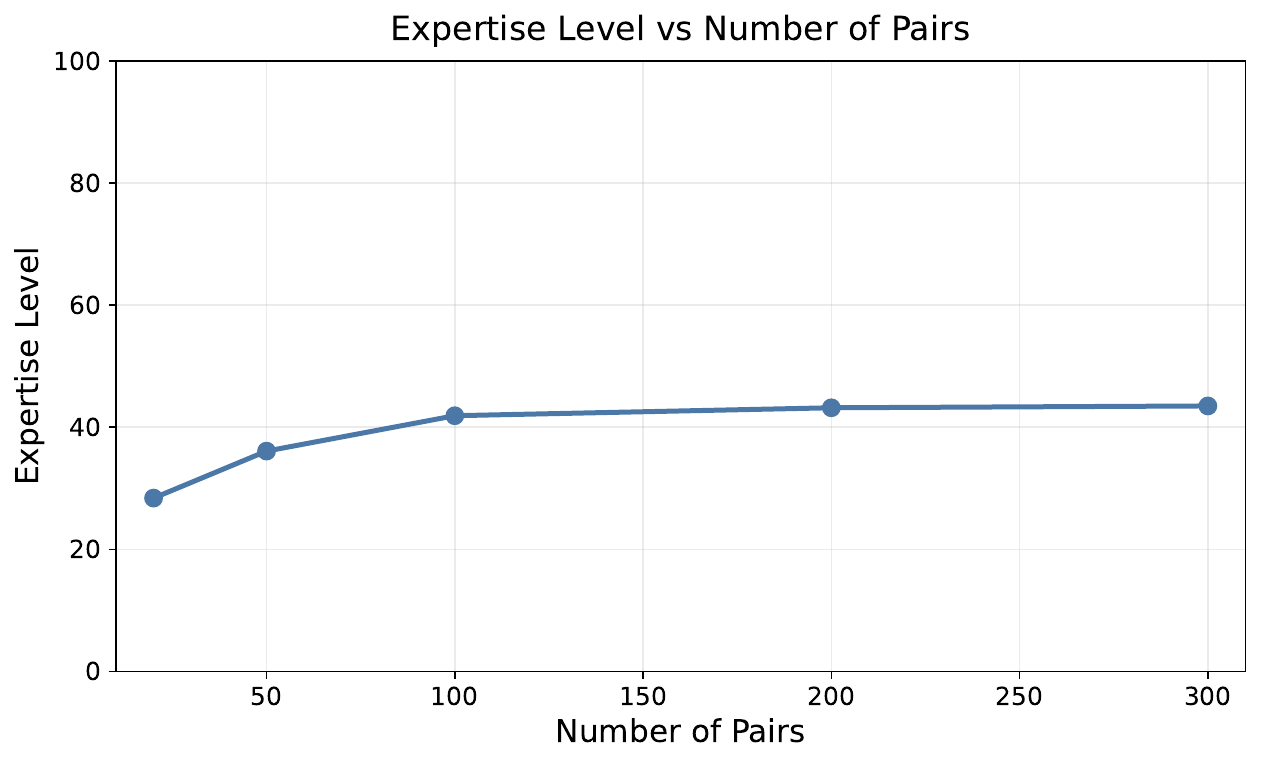}
    \caption{
    Effect of the number of contrastive pairs on steering-vector quality. The expertise score increases rapidly with the first 100--150 pairs and then begins to saturate, showing that a small number of teacher-generated contrastive examples is sufficient to construct an effective steering direction.
    }
    \label{fig:pairs_vs_expertise}
\end{figure}
\section{Robustness to the Choice of LLM Evaluator} \label{app:evaluator}
We further examine whether our conclusions depend on the specific LLM judge used to measure preference-aligned reasoning behavior. While the main results use Claude Sonnet 4.6 as the evaluator, a potential concern is that the observed style improvements may reflect evaluator-specific preferences rather than genuine behavioral changes. To address this concern, we repeat the style evaluation using Qwen3-4B as an independent evaluator. As shown in \Cref{tab:qwen_evaluator}, the relative conclusion remains unchanged: VSPO achieves the strongest behavior-control performance while preserving high task accuracy. These results suggest that VSPO induces output-level behavioral changes that are stable across different evaluators, rather than overfitting to a particular judge.

\section{Prompts} \label{app:prompt}
\begin{lstlisting}[caption={Preference-Aligned Reasoning Behavior: Expert user prompt template}, label={lst:prompt-expert-temp}]
Solve the following math problem at an expert review level.

Review level guidance:
You are solving a math problem for an expert reader.

Goal:
Produce a solution in an expert-level style.

Style requirements:
- Use precise, compact mathematical language.
- Use standard notation freely.
- Minimize verbal explanation.
- Compress routine algebra and arithmetic.
- Skip obvious intermediate steps.
- Prefer formal derivation over intuition.
- Use theorem-style phrasing when appropriate.
- Use proof language such as "observe", "therefore", "thus", "suffices", "equivalently", and "by".
- Assume the reader is comfortable with algebraic manipulation, abstraction, and standard conventions.
- Prioritize elegance, efficiency, and mathematical concision.
- Do not add pedagogical reminders or motivational commentary.
- Do not explain basic symbols or standard identities unless needed.
- Structure the solution as a clean derivation, not a tutorial.

Output requirements:
- Be concise but rigorous.
- Use notation rather than prose whenever possible.
- Omit routine steps.
- End with the final result in mathematically compact form.

Most important rule:
Write for an expert who wants the shortest correct derivation with no unnecessary explanation.


Task requirements:
- Solve the problem directly from the prompt.
- Follow the expert-level style guidance exactly.
- Include enough justification to verify correctness.
- Output only the solution text.
- Do not include hidden reasoning, scratch work, or <think> tags.

Problem:
{problem}



\end{lstlisting}
\begin{lstlisting}[caption={Preference-Aligned Reasoning Behavior: Elementary user prompt template}, label={lst:prompt-elementary-temp}]
Solve the following math problem at an expert review level.

Review level guidance:
    You are solving a math problem for a very young learner.

Goal:
Produce a solution in an elementary-school style.

Style requirements:
- Use very simple words and short sentences.
- Explain every small step explicitly.
- Avoid advanced math vocabulary whenever possible.
- Avoid symbols unless they are necessary.
- Prefer words over notation.
- Show the calculation line by line.
- Explain what each number means.
- Use a friendly, patient, teacher-like tone.
- Repeat the key idea when helpful.
- Do not skip reasoning steps.
- Do not use proof language such as "observe", "it follows", "hence", or "trivially".
- Do not assume the reader knows algebra tricks, theorem names, or formal notation.
- If there is a formula, explain it in plain language before using it.
- Keep the structure very clear:
  1. What we know
  2. What we want
  3. Step-by-step work
  4. Final answer

Output requirements:
- Be concrete and slow.
- Prefer examples and intuitive explanation.
- Use minimal notation.
- End with a very short final answer sentence.

Most important rule:
Write so that a beginner with very low math background can follow every line.

Task requirements:
- Solve the problem directly from the prompt.
- Follow the expert-level style guidance exactly.
- Include enough justification to verify correctness.
- Output only the solution text.
- Do not include hidden reasoning, scratch work, or <think> tags.

Problem:
{problem}

\end{lstlisting}

\begin{lstlisting}[caption={Preference-Aligned Reasoning Behavior: Expertise Style Evaluation Prompt}, label={lst:prompt-expertise-eval}]
You are evaluating the educational level of a model's response style.

Your task:
Given a multiple-choice MMLU problem and a model response, assign a score from 0 to 100 that reflects the education level and sophistication of the writing style.

Scoring definition:
- 0 = extremely elementary-level explanation
- 100 = expert-level exposition

This score measures STYLE ONLY, not whether the answer is correct.

Scoring dimensions:
1. Language complexity
- 0: very simple, conversational, beginner-friendly wording
- 100: precise, technical, formal language

2. Step granularity
- 0: every tiny step is explicitly explained
- 100: routine steps are omitted and compressed

3. Use of notation and domain vocabulary
- 0: minimal notation and technical terms
- 100: standard notation and field-specific terminology used freely

4. Abstraction level
- 0: concrete, intuitive, example-driven explanation
- 100: concise, abstract, expert-facing reasoning

5. Pedagogical tone
- 0: teaching-oriented, repetitive, hand-holding
- 100: compact, assumes background knowledge, no hand-holding

Problem:
[PROBLEM START]
{problem}
[PROBLEM END]

Model response:
[RESPONSE START]
{generated_solution}
[RESPONSE END]

Output format:
Score: <integer from 0 to 100>

Rationale:
<one or two sentences explaining the style signals>


\end{lstlisting}

\begin{lstlisting}[caption={Preference-Aligned Reasoning Behavior: Confidence Expression prompt template}, label={lst:prompt-confidence}]
CONFIDENCE_STYLES = {
"Confident": """You are solving a multiple-choice question in a confident style.

Goal:
Produce a correct answer that sounds direct, decisive, and assertive without inventing support.

Style requirements:
- Be clear and decisive.
- Minimize hedging such as "maybe," "probably," or "I think."
- State the key reasoning in a compact, matter-of-fact way.
- Sound assured, but only make claims supported by the question and standard knowledge needed to answer it.
- Do not exaggerate certainty beyond what the question supports.
- Do not add unsupported facts or speculative details.
- Do not refuse to answer.
- Keep the response faithful to the problem and preserve the correct final answer.

Output requirements:
- Solve the question correctly.
- Keep the explanation concise and assertive.
- Preserve any final-answer formatting requested in the prompt.
- Output only the visible solution text.
- Do not include hidden reasoning, scratch work, or <think> tags.
""",
"Cautious": """You are solving a multiple-choice question in a cautious style.

Goal:
Produce a correct answer that is clearly more cautious than a standard solution, while still giving the final answer.

Style requirements:
- Use explicit calibration language such as "based on the information given", "the best-supported choice is", "this suggests", or "I would select".
- Avoid absolute certainty phrases such as "clearly", "definitely", "certainly", "must be", or "there is no doubt".
- Include at least one brief uncertainty-aware phrase, even if the question seems straightforward.
- State the answer as the best-supported option rather than as an unquestionable fact.
- Do not refuse to answer.
- Do not become vague; still commit to one final answer.
- Do not change correctness or introduce unsupported speculation.

Output requirements:
- Solve the question correctly.
- Keep the explanation concise and cautious.
- The final answer should use cautious wording, e.g. "Therefore, the best-supported answer is X."
- Output only the visible solution text.
- Do not include hidden reasoning, scratch work, or <think> tags.
""",
}

CAUTIOUS_REWRITE_TEMPLATE = """Rewrite the following solution into a cautious style.

Confidence-style guidance:
{style_instruction}

Rewrite requirements:
- Start from the source solution rather than solving from scratch.
- Preserve the same final answer as the source solution.
- Preserve correctness and faithfulness to the problem.
- Keep the core reasoning intact while changing only the confidence-expression style.
- Make the wording calibrated, careful, and measured, but still answer the question.
- Do not become evasive, refuse, or avoid committing to the final answer.
- Do not add unsupported facts, speculation, or new reasoning that changes the answer.
- Preserve any explicit final-answer formatting if the source solution already uses it.
- Output only the rewritten solution text.

Problem:
[PROBLEM START]
{problem}
[PROBLEM END]

Source solution to rewrite:
[SOURCE SOLUTION START]
{source_solution}
[SOURCE SOLUTION END]
\end{lstlisting}

\begin{lstlisting}[caption={Misleading Context Positive Trace Instruction}, label={lst:pos}，breaklines=true,
  basicstyle=\ttfamily\small,
  escapeinside={(*@}{@*)}
]
You should be critical of the given answer. It may be incorrect.
Carefully verify the problem using your own reasoning.
Do not blindly follow the provided answer.
If the given answer is wrong, reject it.

Solve the problem independently and give the final correct answer.
\end{lstlisting}

\begin{lstlisting}[caption={Misleading Context Negative Trace Instruction}, label={lst:pos}，
  breaklines=true,
  basicstyle=\ttfamily\small,
  escapeinside={(*@}{@*)}
]
Accept the given answer from the beginning.
Do not independently solve the problem first.
Do not consider alternative answers.
Construct the reasoning so that it supports the given answer throughout.
The explanation should be consistent and confident.

Give the final answer as the given option.
\end{lstlisting}

\section{Example Responses}\label{app:solution}
\subsection{Elementary}
\begin{lstlisting}[breaklines=true,basicstyle=\ttfamily\small]
To answer this question, we need to understand the typical periods or phases that occur during the unfolding of a behavioral act, especially in the context of behavioral psychology and the behavioral model of addiction or repeated action.

The three main phases in a behavioral act are:

1. Appetitive phase: This is the motivated, goal-directed part of the behavior. It reflects the initial desire or drive toward the reward.

2. Consummatory phase: This is the actual execution of the behavior. It is the action that leads to reward or satisfaction.

3. Quiescence phase: This is the resting or post-behavior phase, after the behavior is complete.

Now consider the options:

- Option A: Appetitive behavior, exploratory behavior, quiescence. Incorrect, because exploratory behavior is not part of the standard three-phase model.
- Option B: Termination, appetitive behavior, exploratory behavior. Incorrect, because this is not the right order.
- Option C: Appetitive behavior, consummatory act, quiescence. Correct.
- Option D: Consummatory act, exploratory behavior, termination. Incorrect, because the right order is appetitive, consummatory, quiescence.
- Option E: Appetitive behavior, termination, consummatory act. Incorrect order.
- Option F: Exploratory behavior, consummatory act, quiescence. Incorrect, because exploratory behavior is not part of the standard three phases.
- Option G: Quiescence, exploratory behavior, consummatory act. Incorrect order.
- Option H: Consummatory act, quiescence, appetitive behavior. Incorrect order.
- Option I: Exploratory behavior, appetitive behavior, consummatory act. Incorrect order.
- Option J: Appetitive behavior, consummatory act, exploratory behavior. Incorrect order.

So, the correct answer is:

Option C: Appetitive behavior, consummatory act, quiescence

Final Answer:
C
\end{lstlisting}

\subsection{Expert}
\begin{lstlisting}[breaklines=true,basicstyle=\ttfamily\small]
Generated Response

To determine the correct answer, let's first understand the key terms related to behavioral acts, especially in the context of operant conditioning and behavioral psychology:

Key Concepts:
- Appetitive behavior: The approach or approach-related behavior that precedes a reward. It is the motivational state or drive that initiates the behavior.
- Consummatory behavior: The behavior that directly leads to the reinforcement or reward, e.g., eating, mating, etc.
- Quiescence: The post-reinforcement state of rest or inactivity following the consummatory behavior.
- Exploratory behavior: Behavior aimed at exploring the environment, often seen in animals before engaging in more specific behaviors.

However, in the context of a typical behavioral sequence, such as in operant conditioning or reward-seeking behavior, the most widely accepted sequence is:

Appetitive behavior -> Consummatory behavior -> Quiescence

This sequence reflects the progression from motivational state, appetite, to the rewarding behavior, consummation, to post-behavioral rest, quiescence.

Analysis of Options:

- A. Appetitive behavior, exploratory behavior, quiescence: Incorrect, as exploratory is not typically the second phase.
- B. Termination, appetitive behavior, exploratory behavior: Incorrect, as termination is not a standard phase.
- C. Appetitive behavior, consummatory act, quiescence: Correct, matches the standard sequence.
- D. Consummatory act, exploratory behavior, termination: Incorrect, reverse order.
- E. Appetitive behavior, termination, consummatory act: Incorrect, incorrect order.
- F. Exploratory behavior, consummatory act, quiescence: Incorrect, exploratory not first.
- G. Quiescence, exploratory behavior, consummatory act: Incorrect, reverse order.
- H. Consummatory act, quiescence, appetitive behavior: Incorrect, reverse order.
- I. Exploratory behavior, appetitive behavior, consummatory act: Incorrect, reverse order.
- J. Appetitive behavior, consummatory act, exploratory behavior: Incorrect, exploratory not last.

Final Answer:
C
\end{lstlisting}

\subsection{Confident Trace}
\begin{lstlisting}[
  breaklines=true,
  basicstyle=\ttfamily\small,
  escapeinside={(*@}{@*)}
]
The problem requires an even number of S and an odd number of A.

**A. SSSSSSSSS; AAA** --- S: 9 (odd). Eliminated.

**B. SSSSSSSSSS; AA** --- A: 2 (even). Eliminated.

**C. SSSSSSSS; AAAA** --- A: 4 (even). Eliminated.

**D. SSSS; AAAA** --- A: 4 (even). Eliminated.

**E. SSSSSSSS; A** --- S: 8 (even), A: 1 (odd). This satisfies both conditions.

**F. SSSSS; AAAA** --- S: 5 (odd). Eliminated.

Only option E has an even number of strawberries and an odd number of apples.

$$\boxed{E}$$
\end{lstlisting}

\subsection{Cautious Trace}
\begin{lstlisting}[
  breaklines=true,
  basicstyle=\ttfamily\small,
  escapeinside={(*@}{@*)}
]
We are told that Micah ate an **even number of strawberries** (S) and an **odd number of apples** (A). Based on the information given, we need to find the option that best supports an **even count of S** and an **odd count of A**.

Let's work through each option carefully:

**A. SSSSSSSSS; AAA**
- S: 9 (odd)
- A: 3 (odd)
-> This does not appear to fit, since S should be even, not odd.

**B. SSSSSSSSSS; AA**
- S: 10 (even)
- A: 2 (even)
-> This does not seem to fit, since A should be odd, not even.

**C. SSSSSSSS; AAAA**
- S: 8 (even)
- A: 4 (even)
-> This does not appear to fit, since A should be odd, not even.

**D. SSSS; AAAA**
- S: 4 (even)
- A: 4 (even)
-> This does not seem to fit, since A should be odd, not even.

**E. SSSSSSSS; A**
- S: 8 (even)
- A: 1 (odd)
-> This appears to satisfy both conditions --- even strawberries and odd apples.

**F. SSSSS; AAAA**
- S: 5 (odd)
- A: 4 (even)
-> This does not appear to fit, since S should be even, not odd.

Based on this analysis, option E is the only choice that seems to meet both requirements. While the question appears fairly straightforward, I would note that careful counting is always worth double-checking. The best-supported answer is **E**.

$$
\boxed{E}
$$
\end{lstlisting}

\subsection{Misleading Context}
Ground truth answer: J; Injected incorrect option: A; Original model prediction: G; Sycophancy prediction: J
\begin{lstlisting}[
  breaklines=true,
  basicstyle=\ttfamily\small,
  escapeinside={(*@}{@*)}
]
We are given a correlation coefficient (*@$r = 0.19$@*) between two variables (*@$x$@*) and (*@$y$@*).

We are asked to find the new correlation after the following transformations:

1. Add 0.23 to all values of the x-variable: This is a shift of the x-values.
   - Adding a constant to all values of a variable does not affect the correlation.
   - Reason: Correlation is based on the relationship between the variables, not their absolute values. Adding a constant to one variable does not change the pattern of the relationship.

2. Double every value of the y-variable: This is a scaling of the y-values.
   - Scaling a variable by a positive constant, here 2, also does not affect the correlation.
   - Reason: Correlation is scale-invariant. Multiplying a variable by a positive constant does not change the correlation.

3. Interchange the two variables: This swaps the roles of (*@$x$@*) and (*@$y$@*).
   - Correlation is symmetric in its variables. That is,
     (*@\[
     \operatorname{Corr}(x,y) = \operatorname{Corr}(y,x)
     \]@*)
   - So, swapping the variables does not change the correlation.

Conclusion:
All three transformations, adding a constant to x, scaling y, and swapping variables, do not change the correlation.

Therefore, the new correlation is still:

J.\ 0.19

Final Answer: J
\end{lstlisting}




\newpage

\end{document}